\documentclass{llncs}

\usepackage{xr-hyper}
\usepackage{hyperref}
\usepackage{float}

\usepackage{amsmath}
\usepackage{amsfonts}
\usepackage{amssymb}
\usepackage{epsf}
\usepackage[left]{eurosym}
\usepackage{multirow}
\usepackage{graphicx}
\usepackage{csvsimple}
\usepackage{color}
\usepackage{geometry}
\usepackage{marginnote}
\usepackage{bbm}
\usepackage{bm}

\usepackage{float}
\usepackage{grffile}
\usepackage{color}
\usepackage{algorithmic}
\usepackage{xcolor}

\include{macrogilles_main}

  \newtheorem{lem}{Lemma}

\newcommand{\subi}{\mathcal{J}}

\begin{document}

\title{Efficient Regularized Piecewise-Linear Regression Trees
}

\author{Leonidas Lefakis\inst{1}, Oleksandr Zadorozhnyi\inst{2}, Gilles Blanchard\inst{2}}

\institute{Zalando Research Lab, Berlin, Germany\\
\email{leonidas.lefakis@zalando.de}
\and
University of Potsdam,
Potsdam, Germany\\
\email{zadorozh@uni-potsdam.de,gilles.blanchard@math.uni-potsdam.de }}

\maketitle

\begin{abstract}

We present a detailed analysis of the class of regression decision tree algorithms which employ a regulized piecewise-linear node-splitting criterion and have regularized linear models at the leaves. From a theoretic standpoint, based on Rademacher complexity framework, we present new high-probability upper bounds for the generalization error for the proposed classes of regularized regression decision tree algorithms, including LASSO-type, and $\ell_{2}$ regularization for linear models at the leaves. Theoretical result are further extended by considering a general type of variable selection procedure. Furthermore, 
 in our work we demonstrate that the analyzed class of regression trees is not only numerically stable but can furthermore be made tractable via an algorithmic implementation, presented herein, as well as with the help of modern GPU technology. Empirically, we present results on multiple datasets which highlight the strengths and potential pitfalls, of the proposed tree algorithms compared to baselines which grow trees based on piecewise constant models. %


\end{abstract}

\section{Introduction}
\label{sec:intro}

Decision trees and random forests remain to be very popular machine learning tools because of their general applicability and considerable
 empirical evidence with regards to their performance on diverse tasks. Both theoretical aspects 
 and practical applications  of such algorithms are active fields of study. 

At the same time in recent years another family of predictors, namely those falling under the umbrella of ``deep-learning'', have met with 
success. Despite recent efforts to attain a better  understanding of the properties 
of this model class, there remains much which is not well understood (from a theoretical perspective). Nonetheless, 
it is clear that the increase in performance of such algorithms is due, at least in part, to their ability
to leverage cutting edge GPU technology to efficiently build complex models using huge datasets. 

Consequently, an important and largely open question is whether other families of predictors can exploit
this technology to increase their performance. In the case of decision trees, recent work \cite{Kont15} 
has in fact shown that decision trees and deep learning are not mutually exclusive and can be combined to create
powerful predictors. Here, however, we are interested in investigating enhancements of the decision trees algorithms, and in particular regression trees
without resorting to building hybrid deep systems.

In particular, we note that an integral component of every tree growing algorithm, to wit  the criterion used for defining the optimal split at each node,  
has remained relatively simple in nature and  is typically based on the squared error of a piecewise constant model.
The motivation behind this simplicity has invariably been linked to the tractability and numerical stability of the underlying algorithm. In the following
we propose a piecewise linear splitting criterion for building a regression tree and show how modern GPU technology makes such complex criteria both tractable
and numerically stable. We furthermore present a detailed theoretical analysis that provides insights into the generalization error of this model class, obtaining high probability upper bounds on the generalization error.
 Finally, we present empirical evidence showing that  proposed algorithm consistently outperforms many
benchmark algorithms which employ the aforementioned piecewise constant model criterion.

\section{Related work}
\label{sec:rel_work}

Least squares regression trees are introduced in the seminal work by Breiman et al.  \cite{breiman1984classification}. As with the large majority
of subsequent tree building algorithms, the proposed CART trees proceed in a top-down greedy manner. Nodes are split based on 
local criteria which in the case of regression trees takes the form of minimizing the variances of the target values in the two
resulting subtrees. This can be interpreted as  minimizing the squared error of a piecewise constant model. 

The M5 algorithm proposed by Quinlan \cite{quinlan1992learning} allows construction of  linear models at the leaves of the resulting tree. The splitting of samples at the internal nodes is chosen so as to minimize the standard deviations on the two sub-populations. Effectively the algorithm builds a regression tree based on a piecewise constant model and exchanges these constant models for more powerful linear models at the leaves. As noted in \cite{landwehr2003logit} due to the similarity in splitting criteria, the CART and M5 algorithms result in similar tree structures.
This similarity not withstanding, M5, and its ``rational reconstruction'' M5$'$ \cite{wang1997inducing} have been shown to outperform CART in practice \cite{wang1997inducing,vogel07scalable}.

Similar to the M5, the HTL algorithm \cite{torgo1997functional}, first builds a CART tree and then replaces the models at the nodes. The main difference in this work from the M5
algorithm is that HTL allows for non-linear regressors at the leaves.
Yet another approach, SECRET \cite{dobra2002secret}, constructs an artificial classification problem, 
and shows that there is merit in building the tree structure using a classification tree algorithm and then finally assigning  linear regression models to each of the leaves.

One characteristic of these and other  approaches, is that they avoid using the final model (usually linear regression) as a criterion when 
optimizing the splits of the internal nodes. M5 for instance, optimizes the split for piecewise constant models in the leaves and only 
inserts the linear models in the leaves post-hoc. 

However, it would be more desirable to build the tree using a split criterion that takes into account the models actually employed in the leaves:
\begin{align}
\label{eq:lrt}
  \sum_{\bx_i \in A} \paren[2]{ y_i- \hat{f}_A(\bx_i;w_A) }^2
  + \sum_{\bx_j \in A^c} \paren[2]{ y_j- \hat{f}_{A^c}(\bx_j;w_{A^c}) }^2,
\end{align}
where $A,A^c$ are the two split subsets and $\hat{f}_A(\cdot;w_A)$ is a linear regressor built on the corresponding subset. The discrepancy of using different
models for optimizing a split and predicting at the leaves, was first pointed out
by Karalic \cite{karalic1992employing}, at the time however, solving the full optimization problem that results from using a piecewise linear model
as part of the splitting criterion was neither tractable nor numerically stable. This intractability and instability was asserted many times in subsequent years \cite{dobra2002secret,Potts2005,Nata}, and in the view of solving
this optimization problem, 
multiple efforts were made to efficiently approximate the above split criterion. Two such approaches \cite{GUIDE,SUPPORT}
use the residuals of a linear regressor at each node to perform a kind of clustering, while another approach \cite{vogel07scalable} 
replaces linear regression with forward selection regression at each node in order to reduce the dimensionality of the models. 


We mention here, for the sake of completeness, that though greedy top-down growing of trees is the most common approach, there exist other approaches, such 
as so-called ``soft'' trees \cite{chipman2010} which predefine the tree architecture and then optimize, in a global manner, the split functions.
Other non-greedy  strategies for learning the split criteria have also been studied in the literature, as for example via multi-linear programming  \cite{benett1994global} and structured prediction  \cite{norouzi2015efficient}. 

We focus however on the specific, sometimes called ``hard'', family of tree algorithms. We show in the following  that despite the increased computational
cost of using the criterion in Equation \eqref{eq:lrt}, modern GPU technology ensures that this approach is both tractable and numerically stable. More importantly
we provide a comprehensive theoretical analysis of this model class, by separating the ``split'' space $\mathbf{\mathcal{X}^{'}}$ from the ``regression'' space
${\mathcal{X}}$ we show that the generalization capabilities of the model class are linked to the dimensionality of the latter as well to the regularization constraints, which is induced by some norm on ${\mathcal{X}}$ (in particular we consider $\ell_{1}-$norm and LASSO, which is well-known for inducing the sparcity). 

The paper is organized as follows: in Section \ref{sec:prob_setup_new} we introduce the notation, present the general framework of Piecewise-Linear Regression Tree (PLRT) algorithm and provide necessary formalism to describe the model class. Section \ref{sec:theor_analysis} contains main theoretical contribution of the paper, namely the upper bound for the Rademacher complexity of the underlying regression trees class and the high probability inequality which controls the deviations of the generalization error. Futhermore, in this Section we provide the important corollaries which relate the different types of penalization (i.e. those which induce the sparsity) on the leaves to the performance of the risk bounds. Section \ref{sec:ctn} is devoted to the question of numerical tractability and stability, while computing the regressor estimates on GPU, whereas Section \ref{sec:emp_eval} reports the empirical results of wide range of experiments, showing practical advantage of using GPU combined with superiority of models returned by PLRT algorithm (in comparison to well-known decision trees baselines). Finally, we conclude in Section \ref{sec:conc}, by highlighting the possible extensions. All the proofs can be found in the Appendix.

\section {Problem setup}
\label{sec:prob_setup_new}

\subsection{Regression algorithm}
\label{subsec: regression_alg}
Let $\mathcal{X} = \mathbb{R}^{d}$, 
equipped with euclidean norm $\norm{\cdot}_{2}$ and $\mathcal{Y} \subset \mathbb{R}$ be some (closed) interval of the real line.
Denote through $\mathbf{S}=\{\bx_{i},y_{i}\}_{i=1}^{n}$ the i.i.d. sample of of size $n$ from some unknown distribution $\mathbb{P}_{{X, Y}}$ over the space $\mathcal{X}\times \mathcal{Y}$. For some input $x \in \mathcal{X}$ we also consider its feature representation in some feature space $\mathcal{X}^{'}$ and denote it as $\psi(x)$. In this work we consider the case where $\mathcal{X}^{'} = \mathbb{R}^{D}$ and denote for an arbitrary point $\bx_{i}$ its representation $\psi(\bx_{i}) := \psi_{i} = \paren{\psi_{i}^{1},\ldots,\psi_{i}^{D}}$ with $d \ll D$, where as usual through $\psi_{i}^{k}$ we denote the $k$-th coordinate of the feature representation $\psi_{i}$.  Therefore, in our work we consider the \textit{extended sample} $\mathbb{S} =\{\bx_{i},y_{i},\psi_{i}\}_{i=1}^{n}$. Lastly, we assume (for the simplicity of the theoretical analysis) that the distribution of $\norm{X}_{2}$ has bounded support in the interval $[0,K]$. We will use the notation $[D]$ for the integer interval
$\set{1,\ldots,D}$.

We investigate both the theoretical properties and empirical performance of a general form piecewise linear regressor with regularization constraints built via the the following regression tree algorithm. 
The proposed algorithm proceeds in
a top-down greedy fashion as is typically done in building a decision tree. At each node in the tree, a split is chosen in the $\mathcal{X}^{'}$ space so
as to minimize the empirical least square errors of linear predictors in the $\mathcal{X}$ space after splitting. 

Let $\subi$ be a subset of indices of the training dataset, corresponding
to instances present at a given tree node. For all $(i,k) \in 
[D]\times \subi$ denote $A_{i,k}=\{ j \in \subi:  \psi_{j}^{i} \geq \psi_{k}^{i} \}$ and $A^{c}_{i,k}= \subi \setminus A_{i,k}$
corresponding to the subsets obtained after a split according to the
$i$-th feature coordinate and threshold $\psi^i_{k}$.
Define the matrix $\bX_{A_{i,k}}$ of dimensions $(|A_{i,k}|,d)$
whose lines are given by $(\bx^t_{l}, l \in A_{i,k})$, and similarly $\bX_{A^{c}_{i,k}}$ of dimensions $(|A^c_{i,k}|, d)$.
Define label vectors $Y_{A_{i,k}}=\paren[1]{y_{l}: l \in A_{i,k}}^t$ and $Y_{A^{c}_{i,k}}=\paren[1]{y_{l}: l \in A^{c}_{i,k}}^t$ of dimension $|{A_{i,k}}|$, $|{A^{c}_{i,k}}|$ respectively.

For every $(i,k) \in [D] \times \subi$ consider the optimal cumulative penalized loss of linear predictors after splitting:

\begin{multline}
\label{eq:L_diff}
L^\lambda_{i,k}= \min_{w_{i,k} \in \mathbb{R}^{d}} \paren{ \norm[1]{\bX_{A^{}_{i,k}}{w}^{}_{i,k}-Y_{A^{}_{i,k}}}_{P}^{2} + \lambda \| w_{i,k} - w_0\|^2_Q}  \\ + \min_{w^c_{i,k} \in \mathbb{R}^{d}} \paren{\norm[1] {\bX_{A^{c}_{i,k}}{w}^{c}_{i,k}-Y_{A^{c}_{i,k}}}_{P}^{2} + \lambda \| {w}^{c}_{i,k} - {w}_0\|^2_Q},
\end{multline}
where $\norm{x}^{2}_{P}:= x^{\top}Px$ and $w_{0}$ may be any vector, though in the following we set $w_0$ to be   the linear regression vector of the parent node. In particular, if for a given leaf $s$ we have computed an optimal regularized least squares solution $w_s^*$; when further splitting this node into two (children) leaves we optimize problems of the form 
$ \min_{w_{i,k} \in \mathbb{R}^{d}} \paren{ \norm[1]{\bX_{A^{}_{i,k}}{w}^{}_{i,k}-Y_{A^{}_{i,k}}}_{P}^{2} + \lambda \| w_{i,k} - w_s^*\|^2_Q}$ and in the similar 
form for $A_{i,k}^{c}$. Thus, the solution at lower nodes are regularized by the propagation of the solutions from higher (parential) nodes. As will be shown in the empirical evaluation, such a regularization is crucial
to the stability of the proposed algorithm. This was in fact a key insight of early experiments, without strong regularization of this form the proposed trees were not able to generalize well.
We note that $L_{i,k}$ can be analytically computed by standard
linear algebra formulas, see Section~\ref{sec:ctn} for a more detailed
discussion of numerical aspects.

As the optimal empirical splitting rule, we choose a pair $(i^{\star},k^{\star})$ minimizing the above: 
$$
(i^{\star},k^{\star}) = \argmin_{(i,k) \in [D] \times \subi} L^\lambda_{i,k}.
$$
For the pair $(i^{\star},k^{\star})$, we split the data accordingly
and get the  subsets $\paren{\bx_\ell,y_\ell}_{\ell \in {A_{i^{\star},k^{\star}}}}$
and $\paren{\bx_\ell,y_\ell}_{\ell \in {A^c_{i^{\star},k^{\star}}}}$
for the two children of the node.
Now, provided that none of the stopping criteria has been reached 
 we apply the splitting algorithm recursively to these two subnodes.
If the stopping criteria is reached, then the linear regression with constraints is performed on the leaves. 



%

\subsection{Regression tree formalism}

For the learning-theoretic study of the algorithm,
we need to define formally the set of possible decision functions that can be output by the learning algorithm. This set must
be data-independent for classical learning-theoretic arguments to apply.

We assume that the total number of leaves is fixed and equal to $\ell$. Let $\mathcal{T}_{\ell}$ be the set of all binary (unlabeled) trees with $\ell$ leaves.  For each tree $T \in \mathcal{T}_{\ell}$, denote by $T^{\circ}$ its interior nodes, and by $\partial T$ its leaves. Consider some interior node $s \in T^{\circ}$ and some splitting rule according to the algorithmic scheme in Section \ref{subsec: regression_alg}. Notice that a pair $(i_{s},t_{s}) \in [D]\times \mathbb{R}$ fully parametrizes the splitting criteria in the form $\mathbbm{1}\paren{{\psi^{i_{s}}(\bx) \geq t_{s}}}$ for some instance $\bx$.
Any split corresponding to a set $A_{i,k}$ which can be obtained by the partition procedure of the PLRT algorithm
can be represented as a pair $(i_s=i,t_s=\psi^{i_s}_k)$ for some $k\in [n]$.
%

Finally, for each leaf $ L_{i} \in \partial T$, $i \in \{1,\ldots,\ell\}$,
the local prediction function at leaf $L_i$ is a linear predictor
$\bx \mapsto \inner{f_i,\bx}$, which we parametrize by the vector $f_i$,
under the constraint $f_{i} \in B$; 
we will consider the constrained classes of the form $ B := \{ \norm{f} \leq W\} \subset \mbr^d$, where $\norm{\cdot}$ is some norm in $\mathbb{R}^{d}$, and track the influence of norm constraints (for some specific choice of norm) on the complexity terms. 


Using the aforementioned notation, we can describe class of regression decision trees as follows: 
\begin{align}
\label{func_class001}
\mathcal{F} := \{f : f = \paren{ T, (i_{s},t_{s})_{s \in T^{\circ}},(f_{k})_{k \in \partial T}}, T \in \mathcal{T}_{\ell},  \paren{i_s,t_s} \in [D]\times \mathbb{R}, f_{k} \in B \}.
\end{align}
The main aim of the next section is to obtain error bounds for statistical performance of the functional class \eqref{func_class001} by means of the deviation of its generalization error. The main technical tool is the concept of Rademacher complexity. 

\section{Theoretical analysis }
\label{sec:theor_analysis}
\subsection{Preliminaries and aim.}
In the following analysis, we obtain high probability upper bounds on the deviation of the \emph{statistical risk} $L(f) = E_{P_{X,Y}}[\ell(f(\bx),y)]$ of 
the model $f \in\mathcal{F}$ 
from \emph{empirical risk} $\hat{L}_{n}(f) = \frac{1}{n}\sum_{i=1}^{n}\ell(f({\bx_{i}}),y_{i})$ uniformly over the model class $\mathcal{F}$. 
We define the \emph{generalization error} of the class $\mathcal{F}$, as $Z := \sup\limits_{f \in \mathcal{F}} \big(L(f) - \hat{L}_{n}(f) \big)$. The main result of this section is a high probability upper bound on the deviations of $Z$ under the different assumptions on the regularization constraints in $B$.  

The expectation and deviation of the generalization error is the typical measure to control the statistical performance of the underlying model class and it was widely studied in \cite{bartlett2002rademacher},\cite{koltchinskii2001rademacher}. In these works the framework of Rademacher complexity is used as the complexity measure for structured regularized risk estimation methods, including kernel methods.

%
%

Let now $\bm{\sigma}:=(\sigma_{i})_{i=1}^{n}$ be a $n$-vector of i.i.d. random variables independent of $\mathbb{S}$, uniformly distributed over $\{-1,+1\}$ (Rademacher random
variables).
For a given loss function $l$, such that $|l(\cdot,\cdot)| \leq C$ and for all $f$ in (an arbitrary) model 
class $\mathcal{G}$, with probability at least $1-\delta$ it holds (see for example \cite{bartlett2002rademacher}):

\begin{equation}
L(f) \leq L_{n}(f)+2\mathfrak{R}_{n}(l \circ \mathcal{\mathcal{G}}) + C\sqrt{\frac{\log \delta^{-1}}{2n}},
\label{eq:gen_error}
\end{equation}
where  
\begin{align}
\mathfrak{R}_{n}(l \circ \mathcal{F}) := E \left[\sup\limits_{f \in \mathcal{F}} \frac{1}{n} \sum\limits_{i=1}^{n}\sigma_{i}l(f(\bx_{i}),y_{i})) \right],
\end{align}
is the \textit{Rademacher complexity} of the model class $\mathcal{F}$, the last expectation
being taken under the product distribution of $(\mbs,\bm{\sigma})$,
where $l \circ \mathcal{G}:= \{ (\bx,y) \mapsto l(f(\bx),y)  | f \in {\mathcal{G}} \}$ is the loss function class associated to $\cG$. 

If we treat the sample $\mathbb{S}$ as fixed,
we introduce the \textit{empirical Rademacher complexity} 

\begin{equation}
\hat{\mathfrak{R}}_{\mbs}(l \circ \mathcal{F}) := E_{\bm{\sigma}}\left[ \sup\limits_{f \in \mathcal{F}} \frac{1}{n}\sum\limits_{i=1}^{n}\sigma_{i}l(f(\bx_{i}),y_{i}))\right],
\label{eq:empirical_rademacher_f}
\end{equation}
so that $\mathfrak{R}_{n}(l \circ \mathcal{F}) = E_{\mathbb{S} \sim P^{\otimes n}_{x,y}}\left[\hat{\mathfrak{R}}_{\mbs}(l \circ \mathcal{F})\right] $. It is also known (see for example \cite{bartlett2005local}) that if the loss-function $\ell$ is $L$-Lipschitz in its second argument,i.e. $\forall y \in \mathbf{Y}$ $ |l(a,y)-l(b,y)| \leq L|a-b| \; \forall a,b \in \mathbb{R}$ then: 
\begin{equation}
\hat{\mathfrak{R}}_{\mbs}(l \circ \mathcal{F}) \leq L \hat{\mathfrak{R}}_{\mbs}(\mathcal{{F}}),
\label{eq:ledoux}
\end{equation}  
and for \eqref{eq:gen_error} we get for all $f$ in the class $\mathcal{F}$: 
\begin{equation}
\label{eq:gen_bound}
L(f) \leq L_{n}(f)+2L \mathfrak{R}_{n} ( \mathcal{F}) + C\sqrt{\frac{\log\delta ^{-1}}{2n}}.
\end{equation}

Now, for the PRLT learning algorithm we consider the squared loss function
$l(y',y)=(y-y')^2$. Recall that from the previous assumptions we have $\norm{\bx}_{2} \leq K$  and $|y| \leq R$ almost surely with respect to ${P}_{X,Y}$. 

Assume also that  for all $\bx$ with $\norm{\bx}_{2}\leq K$ and $f\in \cF$ it holds that $\abs{f(\bx)} \leq F$ (for linear predictors, the constant $F$ depends on the norm-constraints in the set $B$ and will be specified later). Then, we have easily for the squared-loss function:
\begin{align*}
|l(\cdot,\cdot)| \leq (R+F)^2 := L_{R,F,B}.
\end{align*}
Furthermore, we can control $\mathfrak{R}_{n}(\mathcal{F})$ by means of its empirical version $\hat{\mathfrak{R}}_{\mathbb{S}}(\mathcal{F})$, since we observe that changing one point in $\mathbb{S}$ changes $\hat{\mathfrak{R}}_{\mathbb{S}}(\mathcal{F})$ by at most $\frac{F}{n}$, so that by McDiarmid's inequality, with probability at least $1-\delta/2$ it holds: 
\begin{align*}
		\mathfrak{R}_{n}\paren{\mathcal{F}} \leq \hat{\mathfrak{R}_{\mathbb{S}}}\paren{\mathcal{F}} + F\sqrt{\frac{\log{\frac{2}{\delta}}}{2n}}.
\end{align*}

This argument, together with the aforementioned reasoning, implies that in order to obtain high-probability upper bounds for the generalization error $Z$ of model class $\mathcal{F}$, it is sufficient to obtain bounds on its empirical Rademacher complexity based on the sample $\mathbb{S}$. More precisely, with probability at least $1-\delta/2$ it holds for all $f \in \mathcal{F}$ that: 
\begin{align}
\label{eq:risk_ineq}
	L(f) \leq L_{n}(f) + 2L\hat{\mathfrak{R}}_{\mathbb{S}}\paren{\mathcal{F}} + 2LF\sqrt{\frac{\log{\frac{2}{\delta}}}{2n}} + L_{R,F,B}\sqrt{\frac{\log{\frac{2}{\delta}}}{2n}}.
\end{align}
 In the next part we concentrate mainly on the bounds on the (empirical) Rademacher complexity of the model class $\mathcal{{F}}$ of PLRT, but also provide bounds on the true Rademacher complexity.

\subsection{(Empirical) Rademacher complexity and generalization error deviation bound for the class of PLRT}


Recalling  the formal definition of the predictor class in \eqref{func_class001}, we observe that for a fixed a element (tree) $T \in \mathcal{T}_{l}$, and some partition generated by the family $(i_{s},t_{s})_{s \in T^{\circ}} \in S$, we obtain a submodel class $\mathcal{{F}}_{T,\paren{i_{s},t_{s}}_{s \in T^{\circ}}}$ which is a product of the decision models over the leaves, and such that each $\bx_{i} \in \mathbb{S}$ belongs to exactly one leaf $L_{j}$.
Thus, the whole model class $\mathcal{F}$ can be represented as follows: 
\begin{align}
\label{eq:class_str}
\mathcal{F} = \bigcup_{\substack{T \in \cT_l\\(i_{s},t_{s})_{s} \in ([D] \times \mbr)^{\cT^\circ}}}\mathcal{F}_{T,(i_{s},t_{s})_{s\in T^{\circ}}} ,
\end{align}
where we formally write $\mathcal{F}_{T,(i_{s},t_{s})_{s\in \cT^\circ}} = \{f: f = (f_{1},\ldots,f_{\ell}): \forall x \in \mathcal{X},  f(x) = \sum_{j} \inner{f_{j},x}\mathbbm{1}(x \in L_{j}), f_{j} \in \partial T \} $ for each tree $T$ and split family $\paren{i_{s},t_{s}}_{s \in T^{\circ}}$. 
To study the statistical performance of the classes with the union-type  structure \eqref{eq:class_str}
  we make use the of Lemma~2 from \cite{maurer2014inequality} and its corollary, which we provide below for completeness. 

\begin{theorem}[Maurer \cite{maurer2014inequality}]
	\label{thm:sup_bound}
	Let $N \geq 4$ be some natural number and $A_{1},\ldots,A_{N} \subset \mathbb{R}^{n}$ some subsets, such that for a given $A \subset \mathbb{R}^{n}$ we have $A =\bigcup_{i=1}^{N} A_{i}$. Consider $\bm{\sigma} = \paren{\sigma_1,\ldots,\sigma_n}$
        to be a vector of i.i.d. Rademacher variables (i.e.
        uniformly distributed over $\set{-1,1}^n$). Then we have: 
	\begin{align*}
	\ee{}{\sup_{z \in A} \inner{\bm{\sigma},z}} \leq \max_{i=1}^{N}\ee{}{\sup_{z \in A_{i}}\inner{\bm{\sigma},z}} + 4 \sup_{z \in A}\norm{z}\sqrt{\log{N}}.
	\end{align*}
\end{theorem}
Assume we have a finite family of functional classes $\mathcal{F},\mathcal{F}_{1},\ldots,\mathcal{F}_{N}$, such that $\mathcal{F} = \bigcup_{i=1,\ldots,N} \mathcal{F}_{i}$ and a sample $\mathbb{S}$ as before. Denote $A_{j} \subset \mathbb{R}^{n}, A_{j} = \{\paren{f(\bx_{1}),\ldots,f(\bx_{n})}: f \in \mathcal{F}_{j} \}$, $j \in \{1,\ldots, N\}$ (i.e the vector image of the evaluation of function $f \in \mathcal{{F}}_{j}$ on the sample $\mbs$). From the previous result, we have the next corollary (see also \cite{maurer2014inequality}): 
Now, assume we have a finite family of functional classes $\mathcal{F},\mathcal{F}_{1},\ldots,\mathcal{F}_{N}$, such that $\mathcal{F} = \bigcup_{i=1}^{N} \mathcal{F}_{i}$ and a sample $\mathbb{S}$ as before. Denote $A_{j} \subset \mathbb{R}^{n}, A_{j} = \{\paren{f(\bx_{1}),\ldots,f(\bx_{n})}: f \in \mathcal{F}_{j} \}$, $j \in \{1,\ldots, N\}$ (i.e the vector image of the evaluation of function $f \in \mathcal{{F}}_{j}$ on the sample $\mbs$). From the previous result, we have the next corollary (see also \cite{maurer2014inequality}): 
\begin{corollary}
	\label{cor:rad_union}
	For the empirical Rademacher complexity of the class $\mathcal{F}= \bigcup_{i=1,\ldots,N} \mathcal{F}_{i}$ based on the sample $\mathbb{S}$, 
        we have: 
	\begin{align}
	\hat{\mathfrak{R}}_{\mathbb{S}}\paren{\mathcal{F}} \leq \max_{m=1}^{N}\hat{\mathfrak{R}}_{\mathbb{S}}(\mathcal{F}_{m}) + 4\mathcal{M}\sqrt{\frac{\log{N}}{n}},
	\end{align}
	where $\mathcal{M} = \sqrt{\sup_{f \in \mathcal{F}}\frac{1}{n}\sum_{i=1}^{n}f^{2}(\bx_{i})}$.
\end{corollary} 

For the analysis of the class of regression decision trees algorithms \eqref{func_class001}, this means that we can in principle first analyze the
Rademacher complexity of the predictor class for any fixed tree structure and splits, and then pay as a price the $\log$-cardinality of the union appearing in~\eqref{eq:class_str}.
One issue is that (because of the real-valued thresholds $t_s$) this union is not
finite; however using a classical argument we can reduce it to a finite union when
considering the empirical Rademacher complexity.
%
\begin{lemma}
	\label{lem: help_lem01}
	\begin{align*}
          \hat{\mathfrak{R}}_{\mathbb{S}} \paren{\mathcal{F}}
          = \hat{\mathfrak{R}}_{\mathbb{S}} \paren[4]{\bigcup_{\substack{T \in \cT_l\\(i_s,t_s)_{s} \in ([D] \times [n])^{T^\circ}}} \mathcal{F}_{T,(i_{s},\psi_{k_s}^{i_s})_{s\in T^{\circ}}}}
          &\leq \max_{T,\paren{i_{s},t_{s}}_{s \in T^{\circ}}}\hat{\mathfrak{R}}_{\mathbb{S}}(\mathcal{F}_{T,\paren{i_{s},t_{s}}_{s \in T^{\circ}}}) + 4\mathcal{M}\sqrt{\frac{\ell\log{enD}}{n}}, \\
	\end{align*} 
	where $\mathcal{M} = \sqrt{\sup\limits_{f \in \cF}\frac{1}{n}\sum_{i=1}^{n}f^{2}(\bx_{i})}$.
\end{lemma}

With the notation as in Section~\ref{sec:prob_setup_new} we obtain:
%

\begin{theorem}
	\label{thm:main_theorem01}
	With probability at least $1-\delta/2$ it holds uniformly over all $f \in \mathcal{{F}}$: 
	\begin{equation}
	\label{eq:Main_inequality01}
	L(f) - L_{n}(f) \leq  4\paren{R+F} \paren[3]{\max_{T,(i_{s},t_{s})_{s\in T^\circ}}\hat{\mathfrak{R}}_{\mathbb{S}}(\mathcal{F}_{T,(i_{s},t_{s})_{s\in T^\circ}})+ 4\mathcal{M}\sqrt{\frac{\ell \log(enD)}{n}}
          + (R+2F) \sqrt{\frac{\log{ \paren{\frac{2}{\delta}}}}{2n}}}
      \end{equation}
      where we recall that 
      $\abs{f(\bx)} \leq F$ uniformly over all $\bx$ and $\abs{y} \leq R$.
\end{theorem}

In the next section we specify the constraints induced by the set $B$,
and  give explicit bounds for both Rademacher complexity and generalization error for
$\ell_{2}$ and $\ell_{1}$ norm constraints.


\subsection{Different penalty constraints and corresponding bounds}

In what follows we denote $\Sigma = \ee{}{\bx \bx^{\top}}$ the covariance matrix of the random vector $\bx$, and $\hat{\Sigma} = \frac{1}{n} \sum_{i=1}^n \bx_i\bx_i^t$ its
empirical counterpart.

\subsubsection{Euclidean-norm penalty.}
Let $B = \{ \norm{f}_{2} \leq W \}$. Recall that our goal is now to control $\max_{T,(i_{s},t_{s})_{s \in T^{\circ}}} \hat{\mathfrak{R}}_{\mathbb{S}}(\mathcal{F}_{T,(i_{s},t_{s})_{s\in T^{\circ}}})$
for decision trees of total number of leaves $\ell$. The following result holds true.

\begin{lemma}
	\label{eq:l2_bound}
	For a fixed tree structure $T$ with $\ell$ leaves and the splits $(i_{s},t_{s})_{s \in T^{\circ}}$ we have the following 
	upper bounds: 
	\begin{align}
		\hat{\mathfrak{R}}_{\mathbb{S}}\paren{\mathcal{F}_{T,(i_{s},t_{s})_{s\in T^{\circ}}}} &\leq \frac{W\sqrt{2\ell} \sqrt{\sum_{j}\norm{\bx_{j}}_{2}^{2}}}{n} =  W\sqrt{\frac{2 \ell \tr{\hat{\Sigma}}}{n}};\\
		{\mathfrak{R}}_{n}\paren{\mathcal{F}_{T,(i_{s},t_{s})_{s\in T^{\circ}}}} &\leq  W\sqrt{\frac{{2\ell \tr{{\Sigma}}}}{n}}.
	\end{align}	
\end{lemma}

Notice that the upper bound for the empirical Rademacher complexity is data-dependent, but does not depend on the structure of the decision tree with fixed splits.

Plugging in the result of Lemma~\ref{eq:l2_bound} into Lemma~\ref{lem: help_lem01} and consequently plugging the implication of the latter into the general result of Theorem~\ref{thm:main_theorem01} we deduce the following general result. 
\begin{proposition}[Rademacher and Generalization error bound for $\ell_{2}-$norm constraint]
	\label{eq:rad_comp_ell_2}
	Let the function class $\mathcal{F}$ be as given in Equation \eqref{func_class001} with $B=\{\norm{f}_{2} \leq W\}$. Assume $\norm{X}_{2} \leq K$ and $\abs{Y} \leq R$,  $\mathbb{P}_{X}$ and $\mathbb{P}_{Y}$-almost surely.  Then the following upper bounds for the empirical  and true Rademacher complexities of the class $\mathcal{F}$ hold: 
	\begin{align*}
	\hat{\mathfrak{R}}_{\mathbb{S}}\paren{\mathcal{F}} & \leq  
	W \sqrt{\frac{\ell}{n}} \paren{\sqrt{2\tr{\hat{\Sigma}}} + 4\sqrt{\norm[1]{\hat{\Sigma}}_{op}}\sqrt{\log(enD)}};
							\\
	{\mathfrak{R}}_{n}\paren{\mathcal{F}} &\leq W \sqrt{\frac{\ell}{n}} \paren{\sqrt{2\tr{\Sigma}} + 4  \sqrt{\ee[1]{}{\norm[1]{\hat{\Sigma}}_{op}}} \sqrt{\log(enD)} }.
%
\end{align*}
	where we denote the operator norm of a matrix as $\norm{\cdot}_{op}$.
	Furthermore, with probability at least $1-\delta$ it holds: 
	\begin{align*}
			L\paren{f} - L_{n}\paren{f} \leq 4\paren{R+F} \paren{ W\sqrt{\frac{\ell}{n}}\paren{\sqrt{2\tr \hat{\Sigma}} + 4 \sqrt{\norm[2]{\hat{\Sigma}}_{op} }\sqrt{\log\paren{enD}} } + \paren{R+2F}\sqrt{\frac{\log\paren{\frac{2}{\delta}}}{2n}}},
	\end{align*}
	or similarly in data-independent form  (when using bound on Rademacher complexity instead of its data-dependent proxy) we have: 
	\begin{align*}
				L\paren{f} - L_{n}\paren{f} \leq 4\paren{R+F} \paren{  W \sqrt{\frac{\ell}{n}} \paren{\sqrt{2\tr{\Sigma}} + 4  \sqrt{\ee[1]{}{\norm[1]{\hat{\Sigma}}_{op}}} \sqrt{\log(enD)} } + \paren{R+F}\sqrt{\frac{\log\paren{\frac{2}{\delta}}}{2n}}}
	\end{align*}

\end{proposition}

\begin{remark} 
	\label{rem:gen_error}
	Consider the simple case $\Sigma = \mathbbm{I}_{d}$. Then,
in the last bound, the first term scales as $\tr{\Sigma} = d$, but the second term scales like $\e[1]{\norm[1]{\hat{\Sigma}}_{op}}$, which is, due to
classical matrix concentration results \cite{Tropp:15}, up to a $\log(d)$ factor the same as $\norm{\Sigma}_{op} = 1$. This implies: 
\begin{align*}
	L(f) \leq L_{n}(f) & + C_{1} W \sqrt{\frac{\ell}{n}} \paren{ {\sqrt{2d} + 8C_{3}
	\sqrt{\log{d}}\sqrt{\log{enD}}}} +C\sqrt{\frac{\log({\frac{1}{\delta}})}{2n}},
\end{align*}
where $C_{3}$ is some universal numerical constant and $C=(R+WK)^{2}$, $C_{1}=4(R+WK)$. In this case, the first term contributes to the bound exponentially more ($d$ against $\log d$) comparing to the second. Furthermore, if the eigenvalues of $\Sigma$ decay geometrically fast, i.e. $\lambda_{k} \leq \lambda\rho^{k-1}$, with $0 < \rho <1$ for $k \in [d]$, we have $\tr{\Sigma} \leq \frac{\lambda}{1-\rho}$ and $\norm{\Sigma}_{op} = \lambda$ and the bound transforms to: 
\begin{align*}
  L(f) \leq L_{n}(f) & + C_{1}W \sqrt{\frac{\ell}{{n}}} \paren{
                       {\sqrt{\frac{2\lambda}{1-\rho}} + 4C_{3}\sqrt{\lambda}\sqrt{\log{d}}\sqrt{\log{enD}}}} +C\sqrt{\frac{\log({\frac{1}{\delta}})}{2n}}.
\end{align*}
Here the balance between the first and second term depends on the trade-off between the log of dimensionality $d$ and the spectral decay parameter $\rho$. 
\end{remark}
\subsubsection{Lasso-type constraints.}
\label{sec:lasso_type_constraints}
Consider now $B = \{f: \norm{f}_{1} \leq W \}$, where $\norm{\cdot}_{1}$ is the $\ell_{1}-$norm in the space $\mathbb{R}^{d}$. 
 For a fixed tree $T$ and split family $\paren{i_{s},t_{s}}_{s \in T^{\circ}}$, the empirical Rademacher complexity $\hat{\mathfrak{R}}(\mathcal{F}_{T,(i_{s},t_{s})_{s\in T^{\circ}}})$ can be upper bounded as follows:
\begin{lemma}
	\label{lem:rad_lasso_con}
	\begin{align*}
	\hat{\mathfrak{R}}_{\mathbb{S}}(\mathcal{F}_{T,(i_{s},t_{s})_{s\in T^{\circ}}}) &\leq \frac{W\sqrt{2\ell\sum_{i=1}^{n}\norm{\bx_{i}}_{\infty}^{2}}}{{n}}\paren{1+4\sqrt{\log(d)}}.
	\end{align*}
\end{lemma} 
In the similar way, as we used above, plugging in the result of Lemma~\ref{lem:rad_lasso_con} into \ref{lem: help_lem01} we get the following result.

\begin{proposition}[Rademacher and Generalization error bound for $\ell_{1}-$norm constraint]
	\label{prop:Lasso_rad_gen}
	Let the function class $\mathcal{F}$ be as given by Equation~\eqref{func_class001} with $B=\{\norm{f}_{1} \leq W\}$. Then the following upper bounds for the empirical and true Rademacher complexity of the class $\mathcal{F}$ hold: 
	\begin{align*}
	\hat{\mathfrak{R}}_{\mathbb{S}}\paren{\mathcal{F}} &\leq \frac{W\sqrt{\ell}}{\sqrt{n}} \paren{\sqrt{\frac{2}{n}\sum_{j}\norm{\bx_{j}}_{\infty}^{2}}\paren{1 + 4 \sqrt{\log{d}}} + 4\sqrt{\norm[1]{\hat{\Sigma}}_{op}}\sqrt{{\log{enD}}}}; \\ 	{\mathfrak{R}}_{n}\paren{\mathcal{F}} &\leq
	\frac{W\sqrt{\ell}}{\sqrt{n}} \paren{\sqrt{2\ee{}{\norm{\bx_{j}}_{\infty}^{2}}}\paren{1 + 4 \sqrt{\log{d}}} + 4\sqrt{ \ee[1]{}{\norm[1]{\hat{\Sigma}}_{op}}}\sqrt{{\log{enD}}}}.
	\end{align*}
	Furthermore,  in the similar vein to the $\ell_{2}$ regularization case, bounds on the generalization error is obtained by simply applying Theorem~\ref{thm:main_theorem01} and plugging-in the upper bounds on the empirical Rademacher complexity.
\end{proposition}

\subsection{Discussion of the results}
\begin{remark} 
	\label{rem:gen_error_lasso}
	Notice that the second complexity term in the bound on the Rademacher complexity in Proposition~\ref{prop:Lasso_rad_gen} is the same as in the analogue result with the $\ell_{2}$ constraints. For the first complexity term,
it holds $\frac{1}{n}\sum_{j=1}^{n}\norm{\bx_{j}}^{2}_{\infty} \leq \frac{1}{n}\sum_{j=1}^{n}\norm{\bx_{j}}^{2}_{2} = \tr{\hat{\Sigma}} $ and correspondingly $\ee{}{\norm{\bx}_{\infty}^{2}} \leq \tr{\Sigma}$. On the other hand, we have the additional factor of order $\sqrt{\log{d}}$ due to LASSO-type of constraints. Thus, whether LASSO-type regularization gives better bounds for generalization error in comparison to $\ell_{2}$ penalty depends on the magnitude of the ratio $r := \tr{(\Sigma)}/{\ee{}{{\norm{\bx}}_{\infty}^{2}}}$ (or its empirical counterpart $\hat{r} := \tr{(\hat{\Sigma})}/({\frac{1}{n}\sum_{j}\norm{\bx_{j}}_{\infty}^{2}})$\;). Namely, the LASSO-type penalty approach gives a better bound if $r$ or $\hat{r}$ are of order $\log{d}$, where  $d$ is the dimensionality of the regression space.
In the simple case
$\Sigma=\mathbbm{I}_d$ and assuming the coordinates have subgaussian distribution,
we will in fact have
$\tr{\Sigma}=d$ while $\ee{}{\norm{\bx}_{\infty}^{2}} \lesssim \log(d)$.
\end{remark}

\begin{remark}
Theoretical analysis of decision trees ( in particular regression trees) is certanly not new topic and has been considered in various settings before. 
For the case of $1$-dim classification, bounds for generalization error with $0/1$ loss is given in \cite{Mansur:99}.  Furthermore, for multi-class classification the analysis of statistical performance (in terms of bound for generalization error) of a brode class of composite decision trees and forests is given in \cite{Salvo:16}. Although it is formally incorrect to compare these results to our setting, our bounds have the same order of magnitude  in terms of the sample size and number of leaves as in \cite{Mansur:99}. Furthermore,  the error bounds of  scale linearly with the number of leaves, whereas we have sublinear scale ($\sqrt{\ell}$) while having the same order of magnitude in the number of training samples ($1/\sqrt{n}$). In the regression setting the recent work \cite{Lauer:17} author considers the theoretical analysis of the switching regression framework (which includes the framework of regression decision trees). However, this framework does not consider the regularization of the models on the partitioned space. In such setting, using similar toolbox of the Rademacher complexity for statistical error conctrol, in Equation~$(16)$ author obtains risk upper bound which scales linearly with the number of partition cells (which corresponds to the number of leaves in regression trees scenario and comparably worse to our setting). Also, author provides only the majorant term in the upper bound, whereas (as we stressed out before) in our paper we investigate the influence of the smaller order terms (in terms of scaling in logs of dimension and of the norm of covariance matrix). Through usage of more advanced chaining argument, in the section 3.B.2 author shows an upper bound on the Rademacher complexity in the case of linear regression which matches (in the dominant term) the worse case of our bound of Proposition~\ref{eq:rad_comp_ell_2} (when empirical covariance matrix is close to identity).
\end{remark}

\begin{remark}
Notice that our analysis directly imply the bounds for the well-known particular (and much simpler) case of algorithm where the binary regression tree structure are combined with \textit{constant} prediction models on the leaves. This algorithm is know as CART (see original work \cite{Breiman:84},also in \cite{Donoho:97}). Using the formalism we considered in our setting, this type of model class can be described as follows: 
	\begin{align}
	\label{func_class_pc}
		\mathcal{F}_{c} := \{\overline{f} : \overline{f} = \paren{ T, (i_{s},t_{s})_{s \in T^{\circ}},(\overline{f}_{k})_{k \in \partial T}}, T \in \mathcal{T}_{\ell},  \paren{i_s,t_s} \in [D]\times \mathbb{R}, \overline{f}_{k}= \frac{1}{\abs{\partial T_{k}}}\sum_{i}^{}X_{i}\mbi_{ X_{i} \in \partial T_{k}} \},
	\end{align}  
	whereas splits are done by minimizing emprical $\ell_{2}-$error. In the work \cite{Nedelec:03} the statistical performance of pruned CART decision trees (where prunning is perfomed from the maximal tree) of type~\ref{func_class_pc} was analyzed in different setting to ours. Namely, in \cite{Nedelec:03} the oracle-type inequalities for a fixed tree-structure in the frameworks of gaussian and bounded regression have been established. We notice, however, that our results actually applies to much broader setting which as the specific case includes CART. We point out that despite the estimation error will be the same, our PLRT methodology will have smaller empirical error than CART, since instead of greedy error minimization, it employs exact optimization. Since generalization error is bounded by empirical error "plus" estimation error, the second term is bounded by the same quantity for both methods but the first term is smaller for our method, the latter also a better bound on generalization error.
%

\end{remark}
\subsection{Combining PLRT with variable selection procedure}
\label{subsec:variable_select}
In this section we extend the setting by considering the application of some general feature selection rule
(data-dependent, and considered as a black box) before applying the PLRT algorithm, and study its influence on the theoretical generalization performance.    

Let $s \in \mathbb{N}$ be the target number of features, which a learner aims to select. Variable selection for the PLRT learning algorithm can be implemented in different ways, from which we consider the following. At step $0$ the learner can perform selection of $s$ coordinates from the $d-$dimensional vector and build a training sample $\mathbb{S}^{v} = \{\bx_{i}^{v}, \psi_{i},y_{i}\}_{i=1}^{n}$, where $\bx_{i}^{v}$ stands for vector $\bx_{i}$ which has only selected $s$ features given by an index vector $v$, and apply the learning procedure to the sample $\mathbb{S}^{v}$ instead of the original $\mathbb{S}$. Notice that in this case the dimensionality of the feature representation remains the same.
In Appendix~\ref{sec:appendix}  we also consider the alternate
situation of interest
when the feature selection is performed separately
at each leaf of the tree (i.e. the the set of selected features can change from leaf to leaf).


%

Previous theoretical arguments of the deviations control for Rademacher complexity can be extended to the class of decision trees which includes variable selection procedure in the following way.  
For a given number $s$ define the set $\overline{S} = \{a \in \{0,1\}^{d}: \sum a_{i} =s \}$ and for vector $f \in \mathbb{R}^{d}$ denote through $f^{a}$ its $s-$dimensional restriction to the coordinates $i \in \{1,\ldots, d\}$, for which $a_{i} =1$. Define also for $\bw \in \mathbb{R}^{d}$, ${s_{0}}(\bw) = \sharp \{i: \bw_{i} \neq 0 \}$. Given some constraint set $B$ and a number $s \in \mathbb{N}$ we define $B^{s} = \{\bw: \bw \in B, s_{0}\paren{\bw}=s\}$. 

Now consider the extension of the class of PLRT with variable selection (PLRTVS) procedure which can be formally presented as the follows: 
\begin{align}
\label{func_class02}
\mathcal{F}_{s} := \{f : f = \paren{ T, (i_{k},t_{k})_{k \in T^{\circ}},(f_{m})_{m \in \partial T}}, T \in \mathcal{T}_{\ell}, k\in T^{\circ} \paren{i_k,t_{k}} \in [D]\times\mathbb{R}, f_{m} \in B^{s} \}.
\end{align}
We remark that the correction $B^{s}$ instead of $B$ captures all possible choices of the model $f$ on the leaves according to the constraint set $B$ with \textit{arbitrary} $s$ coordinates (which were chosen by a variable selection procedure).

Using a similar argument, for a fixed tree $T \in \mathcal{T}_{l}$, some splits generated by sequence $(i_{k},t_{k})_{k \in T^{\circ}} $ and some set of variables $\overline{s} \in \overline{S}$ we have a model which is a union of the all possible models on the leaves and such that each $\bx_{i} \in \mathbb{S}$ belongs to exactly one leaf $L_{j}$ and is restricted to the coordinates indicated by $\overline{s}$ (and thus of dimension $s$).
The model class $\mathcal{F}$ can be represented as follows: 
\begin{align}
\label{eq:union_select}
\mathcal{F}_{s} = \bigcup_{T,(i_{k},t_{k}),\overline{s}} \mathcal{F}_{T,(i_k,t_k)_{k \in T^{\circ}},\overline{s} \in \overline{S}},
\end{align}
where $\mathcal{F}_{T,(i_{s},t_{s})_{s \in T^{\circ}},\overline{s}} = \{f^{\overline{s}}: f^{\overline{s}} = (f^{\overline{s}}_{1},\ldots,f^{\overline{s}}_{\ell}): \forall x \in \mathcal{X},  f^{\overline{s}}(x) = \sum_{j} \inner{f^{\overline{s}}_{j},x}\mathbbm{1}(x \in L_{j}), f^{\overline{s}}_{j} \in \partial T \}$. 

Now, in order to study the Rademacher complexity of the class $\mathcal{{F}}_{s}$ we can apply the same scheme as before. 
Once more, because of the real-valued thresholds $t_k$ this union is not finite; however using a classical argument same as in the proof of Lemma~\ref{lem: help_lem01} we reduce it to finite union.
\begin{lemma}

	\label{lem:card02}
	
	\begin{align*}
	\hat{\mathfrak{R}}_{\mathbb{S}} \paren{\mathcal{F}_{s}}
	= \hat{\mathfrak{R}}_{\mathbb{S}} \paren[4]{\bigcup_{T,(i_{k},t_{k}),\overline{s}} \mathcal{F}_{T,(i_k,t_k)_{k \in T^{\circ}},\overline{s} \in \overline{S}}}
	&\leq \max_{T,i_{s},t_{s},\overline{s}}\hat{\mathfrak{R}}_{\mathbb{S}}(\mathcal{F}_{T,i_{s},t_{s},\overline{s}}) + 4\mathcal{M}\paren{\sqrt{\frac{\ell\log{enD} + s \log\paren{\frac{de}{s}}}{n}} }, \\
	\end{align*} 
	where $\mathcal{M} = \sqrt{\sup\limits_{f \in \cF}\frac{1}{n}\sum_{i=1}^{n}f^{2}(\bx_{i})}$.
	
\end{lemma}
We demonstrate the influence of the variable selection on the statistical performance for LASSO-type penalization constraints on leaves.
In a very similar vein, we provide data-dependent and distribution-dependent bounds for the empirical Rademacher complexity as well as the generalization error high probability upper bound. Using the analogues of Lemma~\ref{lem: help_lem01} and Theorem~\ref{thm:main_theorem01} for class $\mathcal{{F}}_{s}$ and LASSO-type regularization on $s$ selected variables, we get:

\begin{proposition}
	\label{prop:Lasso_Varsel}
	Let decision class $\mathcal{F}_{s}$ be as given in Equation \eqref{func_class02} with $B=\{\norm{f}_{1} \leq W\}$. Then the following upper bound for the Rademacher complexity of the class $\mathcal{F}_{s}$ is true: 
		\begin{align*}
	\hat{\mathfrak{R}}_{\mathbb{S}}(\mathcal{F}_{s}) &\leq 
	\frac{\sqrt{\ell}W}{\sqrt{n}} \paren[4]{ \sqrt{\frac{2}{n} \sum_{i=1}^{n}\norm{\bx_{j}}^{2}_{\infty}}\paren{1 + 4\sqrt{\log(s)}} + 4 \sqrt{\norm[1]{\hat{\Sigma}}_{op}}\sqrt{\log\paren{neD}}} + \frac{16\sqrt{s}W}{\sqrt{n}}\sqrt{\norm[1]{\hat{\Sigma}}_{op}}\sqrt{\log\paren{\frac{de}{s}}} \\ 
		{\mathfrak{R}}_{n}(\mathcal{F}_{s}) &\leq 
	\frac{\sqrt{\ell}W}{\sqrt{n}} \paren[3]{\sqrt{2\ee{}{\norm{\bx}_{\infty}^{2}}}\paren{1 + 4\sqrt{\log(s)}} + 4 \sqrt{\norm[1]{\hat{\Sigma}}_{op}}\sqrt{\log\paren{neD}}} + \frac{16\sqrt{s}W}{\sqrt{n}}\sqrt{\ee{}{\norm[1]{\hat{\Sigma}}_{op}}}\sqrt{\log\paren{\frac{de}{s}}}
	\end{align*}
	
	Also, with probability at least $1-\delta/2$ we obtain that for all $f \in \mathcal{{F}}_{s}$ it holds:
		\begin{align*}
	L(f) &\leq L_{n}(f)  +  \frac{C_{1}\sqrt{\ell}W }{\sqrt{n}} \paren[4]{\sqrt{\frac{2}{n} \sum_{i=1}^{n}\norm{\bx_{j}}^{2}_{\infty}}\paren{1 + 4\sqrt{\log(s)}} +4 \sqrt{\norm[1]{\hat{\Sigma}}_{op}}\sqrt{\log\paren{neD}}} \\ &+\frac{8C_{1}\sqrt{s}W}{\sqrt{n}}\sqrt{\norm[1]{\hat{\Sigma}}_{op}}\sqrt{\log\paren{\frac{de}{s}}}	+ C_{1}WK \sqrt{\frac{\log(\frac{2}{\delta})}{2n}} + C\sqrt{\frac{\log{\paren{\frac{2}{\delta}}}}{2n}},
	\end{align*}
	where constants $C_{1}=4(R+WK)$, $C=(R+WK)^{2}$ are as before.
\end{proposition}

\subsection{Discussion: generalization properties of variable selection procedure.} 

Firstly, from the results of the Propositions \ref{eq:rad_comp_ell_2}, \ref{prop:Lasso_rad_gen}, \ref{prop:Lasso_Varsel}, we observe that the convergence rates for generalization error are optimal up to the $\log(n)$ factor. Also all bounds on the Rademacher complexities scale linearly with the norm constraint on the prediction vector $f$.

As it was already mentioned before, both Proposition \ref{eq:rad_comp_ell_2} and Proposition \ref{prop:Lasso_rad_gen} imply high probability deviation bounds (both empirical and distribution dependent) for the statistical risk $L(f)$ and can be used for confidence intervals construction. 

Analysing the case of variable selection procedure at root with the LASSO-penalized regressors at leaves, we notice that it reduces the impact of the dimensionality of the underlying linear model class (from $\log{d}$ to $\log{s}$ in the complexity term), adding however an additional factor of $\sqrt{{s\log{\paren{\frac{de}{s}}}}/{n}}$ (where the dependence on the dimension of the regression space is only in $\log(d)$ term) for the generality of the selection rule (which in our framework can be arbitrary).

\begin{remark}
	 It is worth to notice that our analysis can be naturally extended to the case in which at each node the covariates are projected on the lower-dimensional linear manifold (which may still depend on the internal nodes, in which the split is performed). In this case our analysis, as well as the discussions of Remarks \ref{rem:gen_error} and \ref{rem:gen_error_lasso}, hold with $d$ replaced by $m$, thus leading to sharper bounds for this particular subclass of PLRTs. We notice, that in this general case (where splits are data-dependent) in order to control the complexity of the functional class (i.e. apply the Rademacher-type analysis in the same vein as in Proposition~\ref{eq:rad_comp_ell_2} or Proposition~\ref{prop:Lasso_rad_gen} ), we need to have an upper bound on the projected dimension $m$.
\end{remark}

\section{Tractability Analysis}\label {sec:ctn}

As noted in the Introduction, the tractability and stability issues\footnote{We contend that the empirical evaluations presented in Section~\ref{sec:emp_eval} highlight that stability is not an issue. We illustrate 
in Appendix the evaluation of the numerical stability of the underlying algebraic computations} 
through using piecewise linear regressors for split optimization have been long been called into question  (see for example \cite{dobra2002secret,Potts2005,Nata}).
In the following we demonstrate that the full optimization can be efficiently solved. 

\subsection{Computational Complexity }

The optimization problem to be solved for every candidate partitioning $\left(A,A^c\right)$ has the form 
\begin{equation*}
 \left\| \bX_A w - Y_A \right\|^2_P + \lambda \left\| w-w_0\right\|^2_Q,
\end{equation*}
which has an explicit solution 
\begin{equation*}
w_A= (\bX_A^T P \bX_A + \lambda Q)^{-1} (\bX_A^T P Y_A + \lambda Q w_0).
\end{equation*}
As firstly noted by \cite{torgo2002comp}, an efficient algorithm before scanning the thresholds $\psi^{k}_i$, will first sort them and then sequentially move one position in the sorted list at each iteration.
Thus, once it has evaluated partition $A$ it will move on to evaluate the next partition, $A \cup i$ which implies that
\begin{equation}
  w_{A \cup i}
  = \left(  \bX^T_A P \bX_A +   \bx_i ^T P \bx_i  + \lambda Q \right)^{-1} \left( \bX_A^T P Y_A +\bx_i^T P y_i + \lambda Q w_0 \right).
  \label{eq:k}
\end{equation}
 The first factor in the above product (eq.~\ref{eq:k}) can be expressed  
via the Sherman-Morrison formula.
Therefore, the complexity of computing the above inverse is $O(d^2)$, given the inverse $\left(  \bX^T_A  P \bX_A + \lambda Q \right)^{-1}$. 
Thus, given the sorting of the $n$ samples along the $D$ dimensions, which itself has a complexity  $O(Dn\log{n})$ ,  the complexity of finding the optimal split at
each node is $O(nDd^2)$. It depends quadratically on underlying data dimension $d$ and only linearly on $D$ (dimension of the feature space).
The complexity of the proposed method is $O(Dn\log{n} + Dnd^2)$ whereas 
the complexity of analogue model, which uses piecewise constant splits is $O(Dn\log{n} + Dn)$ . 

\subsection{Speeding up}

Recall that $L^\lambda_{i,k}$ is defined as in \eqref{eq:L_diff}.
For each candidate split $i,k$ in order to evaluate $L^\lambda_{i,k}$ we need to compute the two losses
\begin{equation*}
 l^\lambda_{i,k} = \min_{w_{i,k} \in \mbr^{d}} \paren{ \norm[1]{\bX_{A^{}_{i,k}}{w}^{}_{i,k}-Y_{A^{}_{i,k}}}_{P}^{2} + \lambda \| w_{i,k} - w_0\|^2_Q}  
\end{equation*}
\begin{equation*}
r^\lambda_{i,k}= \min_{w^c_{i,k} \in \mbr^{d}} \paren{\norm[1] {\bX_{A^{c}_{i,k}}{w}^{c}_{i,k}-Y_{A^{c}_{i,k}}}_{P}^{2} + \lambda \| {w}^{c}_{i,k} - {w}_0\|^2_Q}.
\end{equation*}
Assume that we have $N$ samples, which should be divided into two sets by a splitting procedure. 
We can significantly speedup the part of the Algorithm which performs tree construction  by noticing that $\forall m>k, l^\lambda_{i,m} \geq l^\lambda_{i,k}$ and $\forall m<k, r^\lambda_{i,m} \geq r^\lambda_{i,k}$.
For all $ i>1$ the algorithm  possess some candidate pair $i_T,k_T$ for which $L^\lambda_{i_T,k_T}$ is the smallest encountered so far. Therefore the algorithm can avoid unneccesary computations
by employing the following strategy: at each iteration for given $j \in [D]$ and $k \in [N]$ instead of computing  $l^\lambda_{j,k}$ and $ r^\lambda_{j,k} $, the algorithm calculates $l^\lambda_{j,k}$ and $ r^\lambda_{j,N-k} $.
If for some $k$, $l^\lambda_{j,k} + r^\lambda_{j,N-k} \geq L^\lambda_{i_T,k_T}$ then $\forall m, k \leq m \leq N-k$ it also holds that $l^\lambda_{j,m} + r^\lambda_{j,m} \geq L^\lambda_{i_T,k_T}$
and those computations can be avoided. We note that this speedup does not reduce the computational complexity of the algorithm, however it is an exact algorithm and as seen
in Fig.~\ref{fig:trac2} can lead to significantly lower training times in practice.

We also illustrate the results for two non-exact versions of the algorithm which use approximations to non-calculated losses to quickly identify non-promising  splits, dimensions of the data and
cut them from the computations. In particular we consider the next approaches: 
\begin{enumerate}
 \item  $\forall m, k \leq m \leq N-k$ and for fixed $j \in [D]$ we approximate 
 \begin{equation*}
L^\lambda_{i,m} \approx l^\lambda_{j,k} + r^\lambda_{j,N-k} + \left(N-2k\right) \left(\min(l^\lambda_{j,k}-\lambda \| {w}_{j,k} - {w}_0\|^2_Q,r^\lambda_{j,N-k}-\lambda \| {w}_{j,N-k} - {w}_0\|^2_Q)\right).
 \end{equation*}
 If for a given $k$ , $L^\lambda_{i,m} \geq L^\lambda_{i_T,k_T}$,then as with the exact algorithm we forgo calculating $l^\lambda_{j,m},r^\lambda_{j,K-m}, \forall m, k \leq m \leq N-k$.

  \item  In a  second approach where  we consider that for all $m, k \leq m \leq N-k$ the following approximation: 
 \begin{equation*}
L^\lambda_{i,m} \approx l^\lambda_{j,k} + r^\lambda_{j,N-k} + \left(N-2k\right) \left(\max(l^\lambda_{j,k}-\lambda \| {w}_{j,k} - {w}_0\|^2_Q,r^\lambda_{j,N-k}-\lambda \| {w^c}_{j,N-k} - {w}_0\|^2_Q)\right).
 \end{equation*}
 As before, if for given $k$ , $L^\lambda_{i,m} \geq L^\lambda_{i_T,k_T}$, we proceed calculating $l^\lambda_{j,m},r^\lambda_{j,N-m}, \forall m, k \leq m \leq N-k$.
\end{enumerate}

\subsection{Training Times in Practice}
As mentioned above, the complexity
of calculating the optimal split, once the thresholds are sorted, is $O(nDd^2)$ which can involve a significant number of computations. In order to empirically evaluate the tractability of the proposed algorithm, 
we  present in Fig.~\ref{fig:trac2} the time which required on a Tesla-K80 GPU to build a full regression tree of depth $10$ for various datasets. As can be seen, a careful algorithmic implementation, in conjunction with the use of GPU technology makes the models tractable despite their considerable complexity.Furthermore, by employing the proposed approximation, the speedup of the Algorithm can be significant and in some cases provide a gain up to an order of magnitude. In the appendix we present an empirical analysis of the effects of the proposed approximations on the accuracy of the trained models and demonstrate that the effects of the approximation can be negligible in many experiments. 
 \begin{figure}
\begin{center}
 \begin{tabular}{|c|c|c|c|c|c|c|c|c|c|}
\hline                         
                                                     Dataset  & KDD04  &  Forest    & KinH & KinM &MNISTOE &MNISTBS   &CT Slice & Air  &Energy   \\                                                      
 \hline 
CART/M5 & 10 &   5 & 0.5 & 0.5
                                                             & 160 & 124 & 75 &0.5&2.5\\                                                  
  \hline 
  PLRT (no speedup) & 6202 &   1454 & 550 & 532
                                                             &  60286& 59583& 14660 &221&2247\\                                                  
  \hline 
  PLRT (exact speedup) & 1803 &   1276 & 481 & 458
                                                             & 29870&33171 & 5928 &84&1878\\                                                  
  \hline 
  PLRT (approx. speedup (min)) & 1014 &   840 & 276 & 236
                                                             & 17379& 16393& 4590 &62&1087\\                                                  
  \hline 
   PLRT (approx. speedup (max)) & 659 &  444 & 148 & 154
                                                             & 9240& 8831& 3222 &55&421\\ 
                                                             \hline                                                           \end{tabular}
                                                          \end{center}
\caption{\label{fig:trac2} Comparison of training times (in sec. ) for constructing a depth-10 tree on the various datasets.}
\end{figure}

\section{Empirical Evaluation}\label{sec:emp_eval}
In this section we present an empirical evaluation
of the generalization capabilities of the proposed PLRT algorithm. In particular, we show results on a number of  different regression tasks comparing different regression tree algorithms, among them
\begin{itemize}
 \item PLRT trees that have been build using a $l_2$ regularized linear regression criterion to split the internal nodes and which similarly at the leaves employ $l_2$ regularization.
 \item  PLRT trees that have been build using a $l_2$ regularized linear regression criterion to split the internal nodes but which employ $l_1$ regularization at the leaves.
 \item M5 trees that are built using a simple piecewise constant model but which employ $l_2$ regularized linear regression  at the leaves. We slightly adapt the algorithm to allow
 it to use all variables to compute the linear regressors at its leaves.
 \item In the few cases where CART trees are competitive we also provide their results. As has been noted CART trees typically empirically underperform when compared to M5  \cite{wang1997inducing,vogel07scalable}.
\end{itemize}
Summary statistics on the datasets is provided in the Table~\ref{tbl:datasets}, we refer reader to Appendix for a detailed presentation of the various tasks. We also notice that in order to distinguish between the regularization parameter for the $l_2$ and $l_1$ penalties we use $\gamma$ in case of $\ell_{2}$ norm and $\lambda$ for LASSO-type. For example in the
plots for PLRT $\gamma=1.0$,$\lambda=0.1$ denotes a tree built by solving $\left\| \bX_A w - Y_A \right\|^2_P + \gamma \left\| w-w_0\right\|^2_Q$, while using LASSO in the leaves
$   \left\| \bX_A w - Y_A \right\|^2 + \lambda \left\| w\right\|$. We note that for the experiments presented here we set $P,Q$ to be the identity matrices of appropriate dimension.
\begin{table}[H]
	\centering
	\caption{Datasets description}
	\label{tbl:datasets}
	\begin{tabular}{llllllll}
		\hline
		\multicolumn{1}{|l|}{Dataset}           &  \multicolumn{1}{l|}{Air} & \multicolumn{1}{l|}{Energy} & \multicolumn{1}{l|}{KDD 2004} & \multicolumn{1}{l|}{Forest} & \multicolumn{1}{l|}{CT-Slice} & \multicolumn{1}{l|}{MNIST}  & \multicolumn{1}{l|}{Kinematic} \\ \hline
		\multicolumn{1}{|l|}{Training set size} & \multicolumn{1}{c|}{6000}      & \multicolumn{1}{c|}{14,803}      &\multicolumn{1}{c|}{30,000}      & \multicolumn{1}{c|}{25,000} & \multicolumn{1}{c|}{27802}   & \multicolumn{1}{c|}{60,000} & \multicolumn{1}{c|}{6000}         \\ \hline
				\multicolumn{1}{|l|}{Testing set size} &  \multicolumn{1}{c|}{3357}      &\multicolumn{1}{c|}{4932}      & \multicolumn{1}{c|}{20,000}      & \multicolumn{1}{c|}{10,000} & \multicolumn{1}{c|}{25298}   & \multicolumn{1}{c|}{10,000} & \multicolumn{1}{c|}{2192}         \\ \hline
		\multicolumn{1}{|l|}{Dimensionality}    & \multicolumn{1}{c|}{12}    & \multicolumn{1}{c|}{28}       &\multicolumn{1}{c|}{77}           & \multicolumn{1}{c|}{54}     & \multicolumn{1}{c|}{383}      & \multicolumn{1}{c|}{784}    & \multicolumn{1}{c|}{32}         \\ \hline	\end{tabular}
\end{table}

\begin{figure}[h]
	\begin{minipage}[h]{0.32\linewidth}
		\center{\includegraphics[width=1\linewidth,scale=0.5]{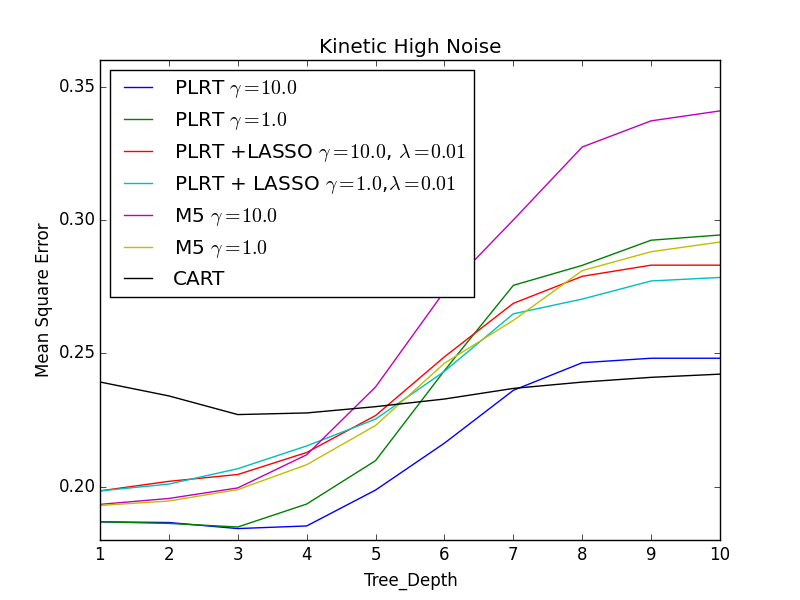}} \\
	\end{minipage}
	\hfill
	\begin{minipage}[h]{0.32\linewidth}
		\center{\includegraphics[width=1\linewidth,scale=0.5]{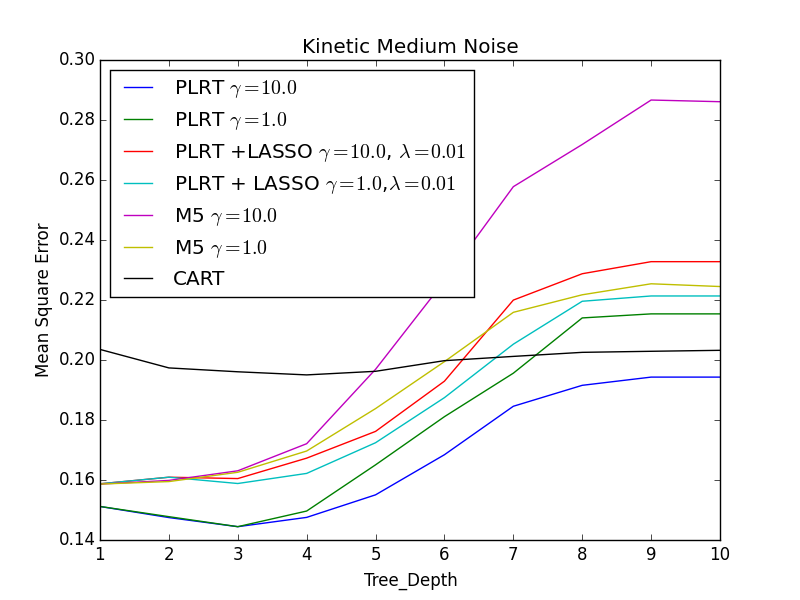}} \\
	\end{minipage}
	\hfill
	\begin{minipage}[h]{0.32\linewidth}
		\center{\includegraphics[width=1\linewidth,scale=0.5]{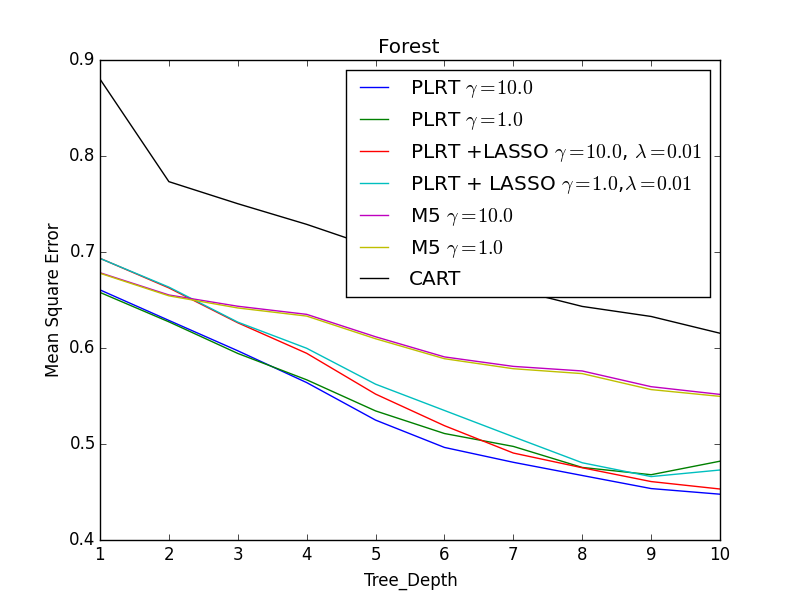}} \\
	\end{minipage}
	\vfill
	\begin{minipage}[h]{0.32\linewidth}
		\center{\includegraphics[width=1\linewidth,scale=0.5]{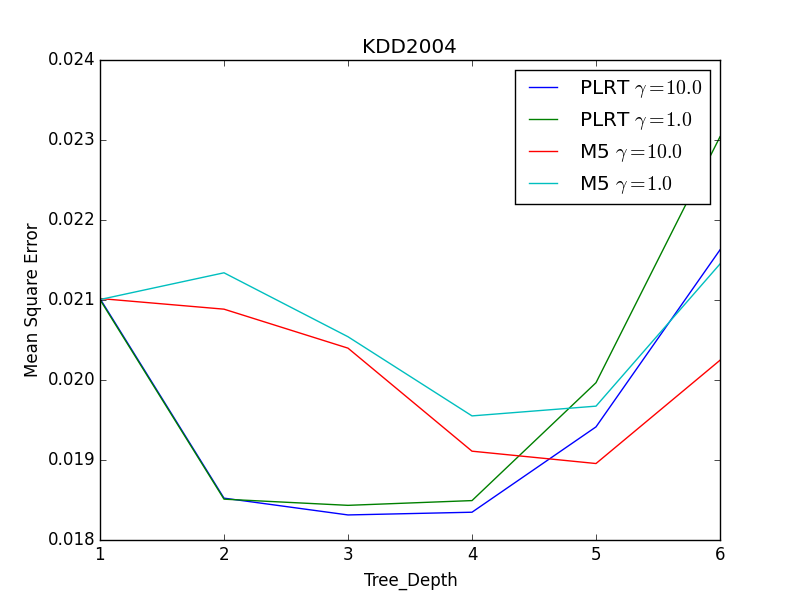}} \\
	\end{minipage}
	\hfill
	\begin{minipage}[h]{0.32\linewidth}
		\center{\includegraphics[width=1\linewidth,scale=0.5]{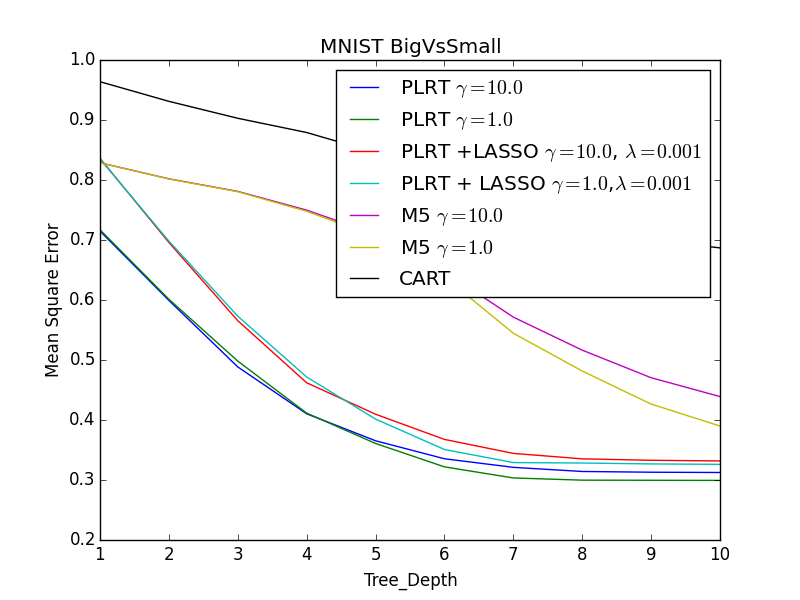}} \\
	\end{minipage}
	\hfill
	\begin{minipage}[h]{0.32\linewidth}
		\center{\includegraphics[width=1\linewidth,scale=0.5]{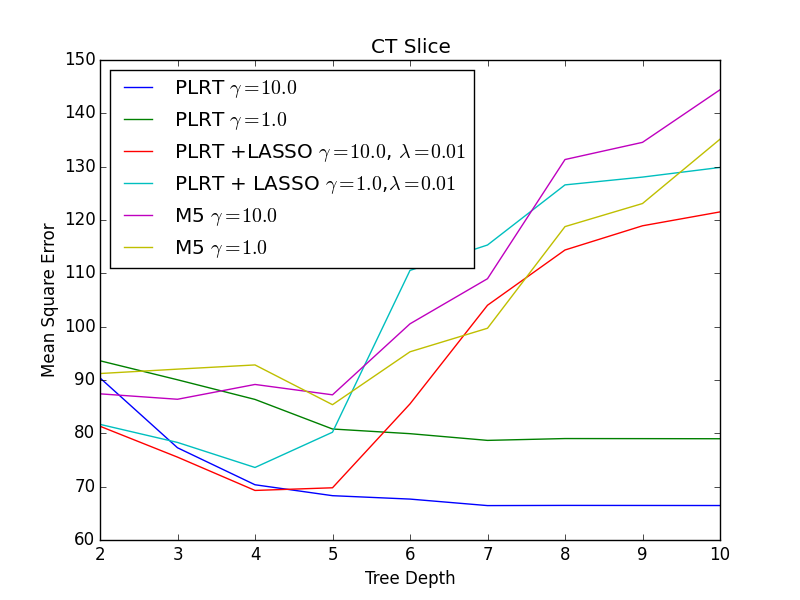}} \\
	\end{minipage}
    \vfill 
	 \begin{minipage}[h]{0.32\linewidth}
		\center{\includegraphics[width=1\linewidth,scale=0.5]{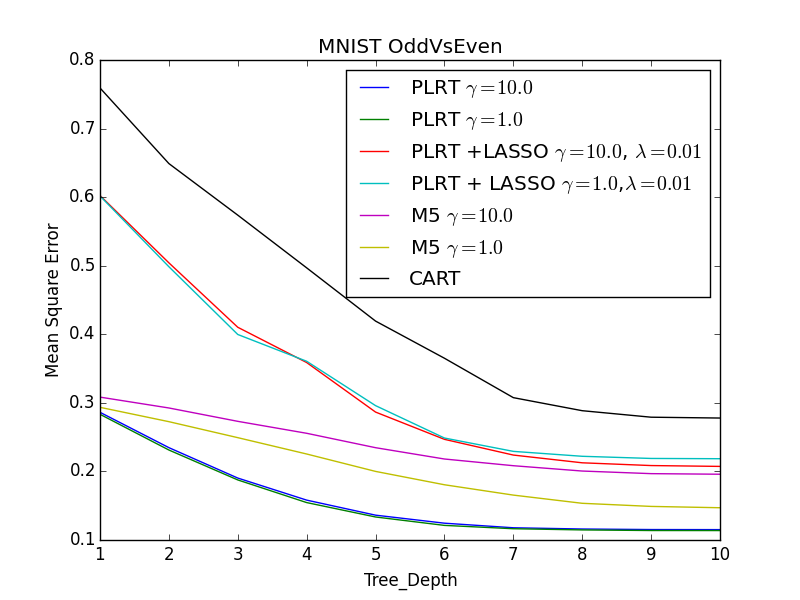}} \\
	\end{minipage}
		 \begin{minipage}[h]{0.32\linewidth}
		\center{\includegraphics[width=1\linewidth,scale=0.5]{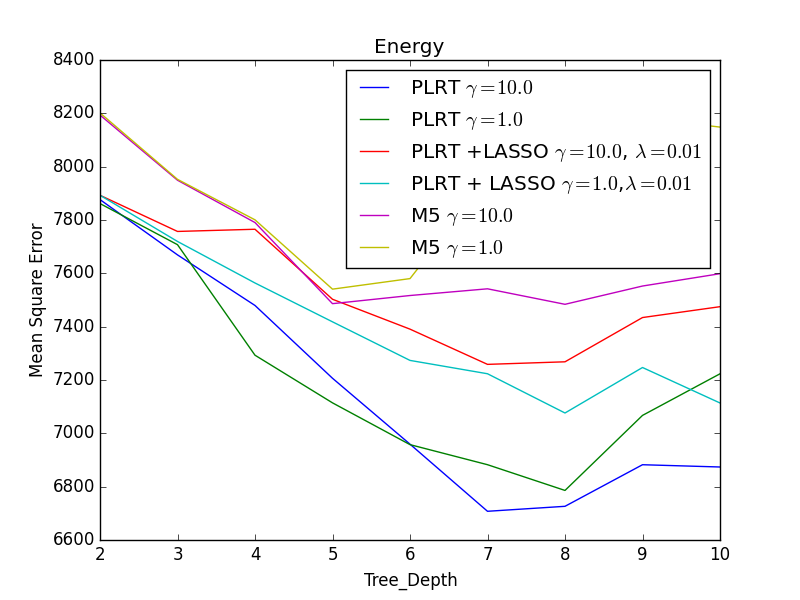}} \\
	\end{minipage}
			 \begin{minipage}[h]{0.32\linewidth}
		\center{\includegraphics[width=1\linewidth,scale=0.5]{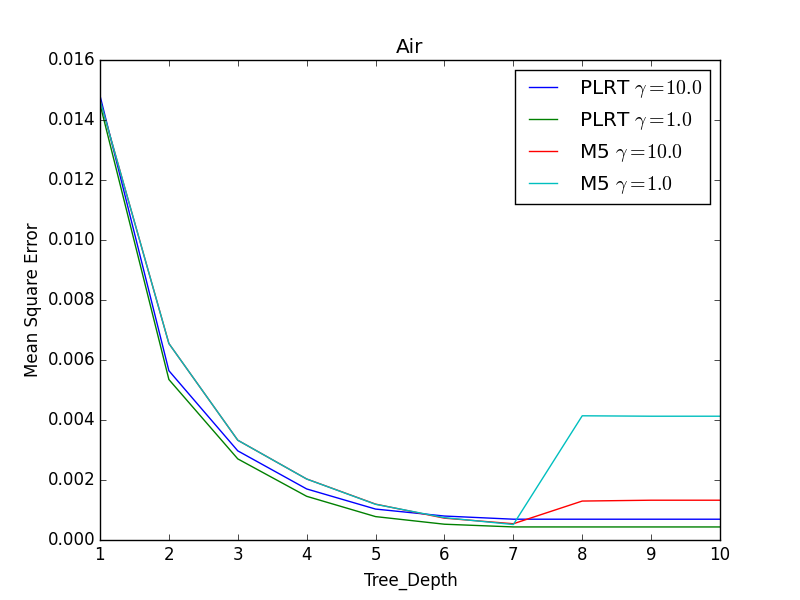}} \\
	\end{minipage}
	\caption{\label{fig:RES}The mean squared error of the 3 predictors for various tree depths. }
\end{figure}
As can be seen in Fig.~\ref{fig:RES}, the proposed algorithm PLRT consistently outperforms the M5 baseline. This is due to the fact that tree structure of PLRT has been  optimized in a greedy fashion and provides at the end complex regression model in the leaves; M5 on the other hand optimizes the tree structure based on a different model than ultimately employed. This allows PLRT to obtain similar or better generalization at smaller tree depths. Though we restrict the results shown here to $\gamma \in \{10.0,1.0\}$, we provide a more detailed analysis 
 of the effects of the $\gamma$ parameter in the Appendix.
 
 The proposed method also outperforms the simpler CART method on all datasets; in fact in most cases CART significantly
underperforms and its results are supressed in order to keep the plot compact. As highlighted by the results on the Kinematic and KDD2004 datasets, PLRT can be prone to overfitting for small sized datasets at larger depths. Nonetheless, in all $3$ tasks involving these datasets PLRT achieves better empirical mean-squared errors when compared to M5. 
We also note that using $\ell_{1}-$penalization
at the leaves does not lead to any added benefit in the empirical performance.  This is due to the discrepancy between the usage of the regularization at the nodes ($l_2$ distance from weight vector of
parent node $\| w-w_0\|$) and at the leaves ($l_1$ regularization $| w |$). A more detailed empirical evaluation of the effects of the $\lambda$ parameter can be found in the Appendix. In two cases (datasets KDD2004 and Air), PLRT + LASSO performs significantly worse than PLRT and M5 and we do not show these results in the plots.

\section{Conclusion}
\label{sec:conc}
In this paper we provided a broad analysis of the class of regression decision trees algorithms with regularized linear prediction models on the leaves.  
From a theoretic perspective, using the celebrated toolbox of Rademacher complexities, we obtained new high probability upper bounds for the generalization error of the
regularized PLRT-algorithms with $\ell_{2}$, $\ell_{1}$ penalty constraints (also including the extension to variable selection setting in the root).
The resulting algorithms are based on the minimization of squared error splitting criteria. We illustrated that the proposed PLRT algorithm ( together with its speedups) are
numerically stable and tractable and can be efficiently implemented with the help of GPU technology.
The empirical study on the diversified types of datasets reveals that the resulting algorithm outperforms the well-known piecewise constant models when
using the mean-squared error metrics with $\ell_{2}$ regularized solutions. From an empirical perspective, we investigate various regularization setups and reach a key insight
regarding the type of regularization needed to avoid overfitting in practice.

\section*{Acknowledgments}
{OZ would like to acknowledge the full support of the Deutsche Forschungsgemeinschaft (DFG) SFB 1294. }


\bibliographystyle{splncs04}
\bibliography{lrt} 

\appendix
\newpage

\section{ Appendix. Proofs of the main results}
\label{sec:appendix}
\subsection{Generalization error bounds of the class of Piecewise Linear Regression tree models}

%

We remind the formal definition used for the class of piecewise-linear regression trees $\mathcal{F}$. 
\begin{align}
\label{eq:regr_trees}
\mathcal{F} := \{f : f = \paren{ T, (i_{s},t_{s})_{s \in T^{\circ}},(f_{k})_{k \in \partial T}}, T \in \mathcal{T}_{\ell},  \paren{i_{s},t_s}_{s\in T^{\circ}} \in [D]\times \mathbb{R}, f_{k} \in B \},
\end{align}
where $\mathcal{T}_{\ell}$ is the set of all trees, $T^{\circ}$ is the set of all internal nodes and $\partial T$ is the set of leaves. Through $\abs{A}$ we denote the cardinality of (finite) set $A$. Also we recall the following representation 
\begin{align}
\label{eq:class_str01}
\mathcal{F} = \bigcup_{\substack{T \in \cT_l\\(i_{s},t_{s})_{s \in T^{\circ}} \in ([D] \times \mbr)^{ \abs{T^\circ}}}}\mathcal{F}_{T,(i_{s},t_{s})_{s\in T^{\circ}}} ,
\end{align}
of the model class as the union of the smaller sub-classes, where we formally write $\mathcal{F}_{T,(i_{s},t_{s})_{s\in T^\circ}} = \{f: f = (f_{1},\ldots,f_{\ell}): \forall x \in \mathcal{X},  f(x) = \sum_{j} \inner{f_{j},x}\mathbbm{1}(x \in L_{j}), f_{j} \in \partial T \} $ for each tree $T$ and split family $\paren{i_{s},t_{s}}_{s \in T^{\circ}}$. 
We present the proofs of the main results of the paper next.
\begin{proof}[Proof of the Lemma \ref{lem: help_lem01} ]
	Observe, that for a fixed coordinate $i \in [D]$ there are $(n-1)$ different splits into two sets that can be obtained by the procedure, implemented in the algorithm. Indeed, if we consider $\psi^{i}$ sorted in decreasing order (w.r.t. samples), then any threshold $t \in [\max_{j} \{\psi_{j}^{i}: \psi_{j}^{i} < \psi_{k}^{i} \}, \psi_{k}^{i}]$ will give exact the same split as $t= \psi_{k}^{i}$. Thus, going iteratively over the (sorted) sample we obtain $n-1$ possibilities for choice of coordinate $k$ (excluding the split, when one subset is empty). The total number of possibilities at one node is therefore $(n-1)D$ and the total number of possible partitions in the fixed tree  $\paren{(n-1)D}^{\ell-1}$ which means, that the infinite union $\mathcal{{F}}$ can be restricted to the finite, when for example choosing the threshold in each step $t_{s} = \psi^{i}_{k}$. Thus, we have representation that $\hat{\mathfrak{R}}_{\mathbb{S}}\paren{\mathcal{F}} = \hat{\mathfrak{R}}_{\mathbb{S}} \paren[4]{\bigcup_{\substack{T \in \cT_l\\(i_s,t_s)_{s} \in ([D] \times [n])^{\abs{T^\circ}}}} \mathcal{F}_{T,(i_{s},\psi_{k_s}^{i_s})_{s\in T^{\circ}}}}$.  Also, the number of different trees with $\ell$ leaves is the $\ell-1$ Catalan's number which is $\frac{1}{l} \binom{2(l-1)}{l-1}$.	Futhermore, by using  simple inequalities $(\frac{k}{e})^k<k!<e(\frac{k}{2})^k$ we obtain $\frac{1}{\ell} \binom{2(\ell-1)}{\ell-1}  \leq \frac{e^{\ell}}{\ell}$ and thus the log-cardinality of the union is bounded by $\log{ \frac{e^{\ell}}{\ell}\paren{(n-1)D}^{\ell-1}} \leq \ell \log{enD}$. Applying Corollary~\ref{cor:rad_union} to  $\hat{\mathfrak{R}}_{\mathbb{S}} \paren[4]{\bigcup_{\substack{T \in \cT_l\\(i_s,t_s)_{s \in T^{\circ}} \in ([D] \times [n])^{\abs{T^\circ}}}} \mathcal{F}_{T,(i_{s},\psi_{k_s}^{i_s})_{s\in T^{\circ}}}}$, and plugging the bound on the log-cardinality of the union in \ref{eq:class_str01} we obtain the claim of the Lemma.
\end{proof} 
\begin{proof}[Proof of Lemma \ref{eq:l2_bound}]
	By linearity of expectation and taking into account, that leaves $(L_{j})_{j=1}^{\ell}$ are disjoint sets we obtain:
	\begin{align*}
	\ee{}{\sup_{f = (f_{1},\ldots,f_{\ell})} \sum_{i=1}^{n}\sigma_{i}f(\bx_{i})} &= \sum_{j=1}^{\ell}\ee{}{\sup_{f_{j} \in B} \inner{f_{j},\sum_{i: \bx_{i} \in L_{j}}^{}\sigma_{i}\bx_{i}}} \\
	& \leq W \sum_{j=1}^{\ell} \ee{}{\norm[2]{\sum_{i: \bx_{i} \in L_{j}}^{}\sigma_{i}\bx_{i}}_{2}} \\
	& \leq W\sum_{j=1}^{\ell}\sqrt{\ee{}{\sum_{i,k: \bx_{i},\bx_{k} \in L_{j}}^{}\sigma_{i}\sigma_{k}\inner{\bx_{i},\bx_{k}}}},
	\end{align*}
	where we used Jensen's inequality in the last line. Since $(\sigma_{i})_{i=1}^{n}$ are independent Rademacher variables, by taking expectation, we obtain: 
	\begin{align*}
	\ee{}{\sup_{f = (f_{1},\ldots,f_{\ell})} \sum_{i=1}^{n}\sigma_{i}f(\bx_{i})} \leq W \sum_{j=1}^{\ell}\sqrt{\sum_{i: \bx_{i} \in L_{j}}\norm{\bx_{i}}^{2}} \leq W \sqrt{\ell} \sqrt{\sum_{i=1}^{n}\norm{\bx_{i}}^{2}},
	\end{align*}
	where 
	in the last inequality we used Jensen's inequality for the function $t \mapsto \sqrt{t}$. The claim of the Lemma now follows from the definition of Rademacher complexity of the fixed class $\mathcal{F}_{T, (i_{s},t_{s})_{s \in T^{\circ}}}$ and the fact that $\sum_{i=1}^{n}\norm{\bx_{i}}^{2} = n \sum_{i=1}^{n}\tr{\frac{1}{n}\bx_{i} \bx_{i}^{\top}} = n \tr{\hat{\Sigma}}$, where $\hat{\Sigma} := \frac{1}{n}\sum_{i=1}^{n}\bx_{i} \bx_{i}^{\top}$ is the empirical covariance matrix.
\end{proof}

\begin{proof}[Proof of Proposition  \ref{eq:rad_comp_ell_2}]
	First of all, we notice that for linear prediction vector $f$ and the constant $\mathcal{M}$ from Lemma \ref{lem: help_lem01} we obtain:
	\begin{align*}
	\mathcal{M} = \sqrt{\sup\limits_{f \in \mathcal{F},f\in B}\frac{1}{n}\sum_{i=1}^{n}\inner{f,\bx_{i}}^{2}}&= W\sqrt {\sup_{f \in \mathcal{F},f\in B}\inner{\frac{f}{W},\frac{1}{n}\sum_{i=1}^{n}\inner[1]{\frac{f}{W},\bx_{i}}\bx_{i}}} \\
	& = W \sqrt{\sup_{f \in \mathcal{F},\norm{f}_{2} \leq 1}\inner{f,\frac{1}{n}\sum_{i=1}^{n}\inner[1]{f,\bx_{i}}\bx_{i}} } \\
	& = W\sqrt{\norm{\hat{\Sigma}}_{op}},
	\end{align*} 
	where as before we denote $\hat{\Sigma} = \frac{1}{n}\sum_{i=1}^{n}\bx_{i}\bx_{i}^{\top}$ to be the empirical covariance matrix estimated from sample $\mathbb{S}$. 
	Considering the fact that the (empirical) Rademacher complexity of one fixed (sub)~class $\mathcal{F}_{T,\paren{i_{s},t_{s}}_{s \in T^{\circ}}}$ does not depend on the structure and partition (splitting family) of the tree, by applying Lemma~\ref{lem: help_lem01} together with the claim of Lemma~\ref{eq:l2_bound} for Rademacher complexity of $\hat{\mathfrak{R}}_{\mathbb{S}}\paren{\mathcal{F}_{T,(i_s,t_{s})_{s \in T^{\circ}}}}$ we deduce:
	\begin{align*}
	\hat{\mathfrak{R}}_{\mathbb{S}}\paren{\mathcal{F}} \leq \frac{2^{\frac{1}{2}}\sqrt{\ell}W\sqrt{\tr{\hat{\Sigma}}}}{\sqrt{n}} + 4W\sqrt{\norm{\hat{\Sigma}}_{op}}\sqrt{\ell\frac{\log{enD}}{n}}.
	\end{align*}
	For the true Rademacher complexity, we simply take the expectation of $\hat{\mathfrak{R}}_{\mathbb{S}}\paren{\mathcal{F}}$, use Jensen's inequality and exchange trace and expectation (due to their linearity) in the first summand; thus we get: 
	\begin{align}
	\label{eq:rad_bound_true02}
	\begin{aligned}
	\mathfrak{R}_{n}\paren{\mathcal{F}} & \leq \frac{2^{\frac{1}{2}}\sqrt{\ell}W\sqrt{\ee{}{\tr{\hat{\Sigma}}}}}{\sqrt{n}} + 4W\sqrt{\ee{}{\norm[1]{\hat{\Sigma}}_{op}}}\sqrt{\ell\frac{\log{2enD}}{n}} \\
	&= \frac{2^{\frac{1}{2}}\sqrt{\ell}W\sqrt{\tr{\Sigma}}}{\sqrt{n}} + 4W\sqrt{\ee{}{\norm[1]{\hat{\Sigma}}_{op}}}\sqrt{\ell\frac{\log{2enD}}{n}}.
	\end{aligned}
	\end{align}
	Furthermore, since $f(\bx) \leq \norm{f}_{2}\norm{\bx}_{2} \leq WK$, putting this together the upper bound for the empirical Rademacher complexity  and the bounds for Lipschitz constant, constants $F$ and $L_{R,B,F}$ into the Theorem \ref{thm:main_theorem01} we obtain with probability at least $1-\delta/2$:
	\begin{align*}
	L(f) \leq L_{n}(f) & + 4\paren{R+WK} \paren{\frac{2^{\frac{1}{2}}\sqrt{\ell}W\sqrt{\tr{\hat{\Sigma}}}}{\sqrt{n}} + 4W\sqrt{\norm{\hat{\Sigma}}_{op}}\sqrt{\ell\frac{\log{2enD}}{n}}}\\
	& + 4\paren{R+WK}WK \sqrt{\frac{\log(\frac{2}{\delta})}{2n}} + \paren{R+WK}^{2}\sqrt{\frac{\log{\frac{2}{\delta}}}{2n}}.
	\end{align*}
	Finally, to obtain the alternative upper bound for the generalization error we just use the distribution dependent result for Rademacher complexity and get: 
	\begin{align*}
	L(f) \leq L_{n}(f) & + 4\paren{R+WK}\frac{W\sqrt{\ell}}{\sqrt{n}} \paren{{2^{\frac{1}{2}}\sqrt{\tr{\Sigma}}} + 4\sqrt{\ee{}{\norm[1]{\hat{\Sigma}}_{op}}}\sqrt{{\log{2enD}}}}
	+ \paren{R+WK}^{2}\sqrt{\frac{\log{\frac{2}{\delta}}}{2n}}.
	\end{align*}

\end{proof}

\begin{proof}[Proof of Lemma \ref{lem:rad_lasso_con} (LASSO-type constraints)]

	Repeating the similar arguments as in the case with $\ell_{2}-$penalty we obtain: 
	\begin{align*}
	\hat{\mathfrak{R}}_{\mathbb{S}}\paren{\mathcal{F}_{T,(i_{s},t_{s})_{s\in T^{\circ}}}} = \frac{2}{n}\ee{}{\sup_{f = (f_{1},\ldots,f_{\ell})} \sum_{i=1}^{n}\sigma_{i}f(\bx_{i})} &= \sum_{j=1}^{\ell}\frac{2}{n}\ee{}{\sup_{f_{j} \in B} \inner{f_{j},\sum_{i: \bx_{i} \in L_{j}}^{}\sigma_{i}\bx_{i}}} \\
	& = \sum_{j=1}^{\ell}\frac{r_{j}}{n}\hat{\mathfrak{R}}_{\mathbb{S}}\paren{\hat{\mathcal{F}}_{j}},
	\end{align*}
	where $r_{j}=\abs{L_{j}}$ and $\hat{\mathcal{F}}_{j}$ is the empirical Rademacher complexity of class of functions, such that $\norm{f}_{1} \leq W$ computed on the $r_{j}$ samples. Applying Proposition 3 from \cite{maurer2012structured} with $\mathcal{P}$ being the set of orthogonal projectors on the coordinates and $\norm{f}_{\mathcal{P}} = \norm{f}_{1}$, after rescaling the function $f$ by $W$ we obtain:
	:
	\begin{align*}
	\hat{\mathfrak{R}}_{\mathbb{S}}\paren{\hat{\mathcal{F}}_{j}} \leq \frac{W\sqrt{2\sum_{i: \bx_{i} \in L_{j}} \norm{\bx_{i}}^{2}_{\infty}}}{r_{j}} \paren{1+4\sqrt{\log{d}}},
	\end{align*}
	which implies that 
	\begin{align*}
	\hat{\mathfrak{R}}_{\mathbb{S}}\paren{\mathcal{F}_{T,(i_{s},t_{s})_{s\in T^{\circ}}}} \leq \frac{1}{n}W\paren{1+4\sqrt{\log{d}}}\sum_{j=1}^{\ell}\sqrt{2\sum_{i: \bx_{i} \in L_{j}} \norm{\bx_{i}}^{2}_{\infty}} \leq \frac{2^{\frac{1}{2}}\ell^{\frac{1}{2}}W\sqrt{\sum_{i}\norm{\bx_{i}}^{2}_{\infty}}}{{n}}\paren{1+4\sqrt{\log{d}}},
	\end{align*}
	where in the last inequality we used Jensen's inequality in the similar fashion as for $\ell_{2}-$norm constraints for the function $t \mapsto \sqrt{t}$. 
	The statement is therefore proved.
	
\end{proof}
\begin{proof}[Proof of Proposition \ref{prop:Lasso_rad_gen} ]
	The general scheme of the proof remains to be the same as in the proof of Proposition \ref{eq:rad_comp_ell_2}. Namely, since we have that for any $f \in \mathbb{R}^{d}$ $\norm{f}_{2} \leq \norm{f}_{1}$ then, for the constant $\mathcal{M}$ we obtain: 
	\begin{align*}
	\mathcal{M} = W\sqrt {\sup_{f \in \mathcal{F},f\in B}\inner{\frac{f}{W},\frac{1}{n}\sum_{i=1}^{n}\inner[1]{\frac{f}{W},\bx_{i}}\bx_{i}}} = W \sqrt{\sup_{f \in \mathcal{F},\norm{f}_{1} \leq 1}\inner{f,\frac{1}{n}\sum_{i=1}^{n}\inner[1]{f,\bx_{i}}\bx_{i}} } \leq W\sqrt{\norm{\hat{\Sigma}}_{op}}.
	\end{align*}
	Thus, using Lemma \ref{lem: help_lem01} with simply
	$\mathcal{M} \leq W \sqrt{\norm{{\hat{\Sigma}}}_{op}}$ and bound on the log-cardinality of the union, together with Lemma \ref{lem:rad_lasso_con} results into: 
	\begin{align*}
	\hat{\mathfrak{R}}_{\mathbb{S}}\paren{\mathcal{F}} \leq 
	\frac{2^{\frac{3}{2}}\ell^{\frac{1}{2}}W\sqrt{\sum_{i=1}^{n}\norm{\bx_{i}}^{2}_{\infty}}}{{n}}\paren{1+4\sqrt{\log{d}}} + 8W\sqrt{\norm[1]{\hat{\Sigma}}_{op}}\sqrt{\ell\frac{\log{2enD}}{n}}. 
	\end{align*} 
	For the true Rademacher complexity in the similar vein to the case of $\ell_{2}$ penalty we get: 
	\begin{align}
	\label{eq:rad_bound_true_lasso}
	\begin{aligned}
	\mathfrak{R}_{n}\paren{\mathcal{F}} & \leq \frac{2^{\frac{1}{2}}\sqrt{\ell}W\sqrt{\sum_{j}\ee{}{\norm{\bx_{j}}_{\infty}^{2}}}}{{n}} + 8W\sqrt{\ee{}{\norm[1]{\hat{\Sigma}}_{op}}}\sqrt{\ell\frac{\log{2enD}}{n}} \\
	& = \frac{2^{\frac{1}{2}}\sqrt{\ell}W\sqrt{\ee{}{\max_{k}\paren[1]{\bx^{k}}^{2}}}}{\sqrt{n}}\paren{1+4\sqrt{\log d}} + 8W\sqrt{\ee{}{\norm[1]{\hat{\Sigma}}_{op}}}\sqrt{\ell\frac{\log{2enD}}{n}},
	\end{aligned}
	\end{align}
	which proves the bound for Rademacher complexity.
	The result for the generalization error follows immediately from Theorem \ref{thm:main_theorem01}.
\end{proof}
\subsection{Proofs for variable selection results}
In this part we provide the results and sketch the proofs for the control of Rademacher complexities and generalization error upper bounds for the extension of PLRT algorithm to the variable selection criteria.
and the following upper bound is true:
\begin{lemma}
	\label{lem:card02}	
	\begin{align*}
	\hat{\mathfrak{R}}_{\mathbb{S}} \paren{\mathcal{F}_{s}}
	&= \hat{\mathfrak{R}}_{\mathbb{S}} \paren[4]{\bigcup_{T,(i_{k},t_{k}),\overline{s}} \mathcal{F}_{T,(i_k,t_k)_{k \in T^{\circ}},\overline{s} \in \overline{S}}}
	\\	&\leq \max_{T,i_{s},t_{s},\overline{s}}\hat{\mathfrak{R}}_{\mathbb{S}}(\mathcal{F}_{T,i_{s},t_{s},\overline{s}}) + \\& + 4\mathcal{M}\paren{\sqrt{\frac{\ell\log{enD} + s \log\paren{\frac{de}{s}}}{n}} }, \\
	\end{align*} 
	where $\mathcal{M} = \sqrt{\sup\limits_{f \in \cF}\frac{1}{n}\sum_{i=1}^{n}f^{2}(\bx_{i})}$.
	
\end{lemma} 
\begin{proof}[Proof of Lemma \ref{lem:card02}]
	Repeating the argument of the proof of Lemma \ref{lem: help_lem01}, for finite $n$, we can restrict the model class $\mathcal{F}_{s}$ to a finite union, encounting only splits with $t_{s} = \psi^{i}_{k}$ for all $i \in [D]$, $k \in [n]$. Notice, the number of possible choices of $s$ coordinates from $d$ is $\binom{d}{s}$ we derive, that the log-cardinality of the  (finite) union is upper bounded by $\ell \log{enD} + \log{\binom{d}{s}}$
	Now using the inequality $s! \geq \paren{\frac{s}{e}}^{s}$, from which we can derive that: 
	\begin{align}
	\binom{d}{s} = \frac{\prod_{j=0}^{s-1}\paren{d-j}}{s!} \leq \frac{d^{s}}{s!} \leq \paren{\frac{de}{s}}^{s}.
	\end{align}
	Then, the result is follows, when taking the logarithm and summing up the terms and proceeding directly as in Lemma \ref{lem: help_lem01}. 
\end{proof}

%
%
%
\subsection{Multiple feature selection.}

\label{S-mfs}

The feature selection procedure in the regression space $\mathcal{X}$ can be also performed both in the internal nodes for finding the penalized least squares solution \textit{and} additionally on the leaves to build the final regressors. More precisely, at each internal node $1,\ldots, \ell-1$ we select $s$ features among $d$ from the available dataset $\bX$, find the best split by building the (penalized) cumulative least squares loss, based on selected $s$ features. After the tree has been build, at each leaf ($1,\ldots,\ell$) select $s$ (possibly different) features from $d$ and compute the regression solutions. Thus, feature selection procedure is performed $2\ell-1$ times. Formally, class of decision allows the union representation as follows: 
\begin{align}
\label{eq:class_str_all}
\mathcal{F}_{\ell,sel} = \bigcup_{\substack{T \in \cT_l\\(i_{k},t_{k})_{k} \in ([D] \times \mbr)^{\cT^\circ} \\ \paren{\overline{s}}_{p} \in \overline{S}^{\otimes 2\ell-1}}}\mathcal{F}_{T,\paren{(i_{k},t_{k})_{k\in T^{\circ}}}, \paren{\overline{s}}_{p}}.
\end{align}
The reasoning in Section \ref{subsec:variable_select} (with the only feature selection procedure at the root) can be straightforwardly extended to the case  when we use feature selection at each internal node and we obtain the following analogues of Lemma~\ref{lem:card02} and Proposition~\ref{prop:Lasso_Varsel}. Firstly we formulate the quite analogues version of Lemma \ref{lem:card02} with multiple splits.

\begin{lemma}
	\label{lem:card03}
	\begin{align*}
	\hat{\mathfrak{R}}_{\mathbb{S}} \paren{\mathcal{F}_{\ell,sel}}
	= \hat{\mathfrak{R}}_{\mathbb{S}} \paren[4]{\bigcup_{\substack{T \in \cT_l\\(i_{k},t_{k})_{k} \in ([D] \times \mbr)^{\cT^\circ} \\ \paren{\overline{s}}_{p} \in \overline{S}^{\otimes 2\ell-1}}}\mathcal{F}_{T,\paren{(i_{k},t_{k})_{k\in T^{\circ}}}, \paren{\overline{s}}_{p}}}
	&\leq \max_{T,i_{s},t_{s},\overline{s}}\hat{\mathfrak{R}}_{\mathbb{S}}(\mathcal{F}_{T,i_{s},t_{s},\paren{\overline{s}}_{p}}) \\
	+ 4\mathcal{M}\paren{\sqrt{\frac{\ell\log{enD} + 2s\ell \log\paren{\frac{de}{s}}}{n}} },
	\end{align*} 
	where $\mathcal{M} = \sqrt{\sup\limits_{f \in \cF}\frac{1}{n}\sum_{i=1}^{n}f^{2}(\bx_{i})}$.
\end{lemma}

\begin{proof}[Proof of Lemma \ref{lem:card03}]
	The proof is done in the very same way as in the Lemma \ref{lem:card02}, noticing that now we choose $s$ variables out of $d$ \textit{at each node} which increases the factor of the union's cardinality to $\binom{d}{s}^{2\ell-1}$. Rest of the proof follows the same argument as before.
\end{proof}
\begin{proposition}
	\label{prop:Lasso_Varsel_all}
	Let function class $\mathcal{F}_{\ell,sel}$ be as given in the equation \eqref{eq:class_str_all} with $B=\{\norm{f}_{1} \leq W\}$. Then the following upper bound for the Rademacher complexity of the class $\mathcal{F}_{\ell,sel}$ is true: 
	\begin{align*}
	\hat{\mathfrak{R}}_{\mathbb{S}}(\mathcal{F}_{\ell,sel}) \leq 
	\frac{\sqrt{\ell}W}{\sqrt{n}} \paren[4]{2^{\frac{1}{2}} \sqrt{\frac{1}{n} \sum_{i=1}^{n}\norm{\bx_{j}}^{2}_{\infty}}\paren{1 + 4\sqrt{\log(s)}} + 4 \sqrt{\norm[1]{\hat{\Sigma}}_{op}}\sqrt{\log\paren{neD}} + \frac{8\sqrt{s}}{\sqrt{n}}\sqrt{\norm[1]{\hat{\Sigma}}_{op}}\sqrt{\log\paren{\frac{de}{s}}}} .
	\end{align*}
	Also, with probability at least $1-\delta/2$ we obtain that for all $f \in \mathcal{{F}}_{\ell,sel}$ it holds:
	\begin{align*}
	L(f) - L_{n}(f)&\leq   + C_{1} \paren[3]{ \frac{\sqrt{\ell}W}{\sqrt{n}} \paren[2]{2^{\frac{1}{2}} \sqrt{\frac{1}{n} \sum_{i=1}^{n}\norm{\bx_{j}}^{2}_{\infty}}\paren{1 + 4\sqrt{\log(s)}} + 4 \sqrt{\norm[1]{\hat{\Sigma}}_{op}}\sqrt{\log\paren{neD}}} }\\ &+\frac{8C_{1}\sqrt{s\ell}W}{\sqrt{n}}\sqrt{\norm[1]{\hat{\Sigma}}_{op}}\sqrt{\log\paren{\frac{de}{s}}}	+ C_{1}WK \sqrt{\frac{\log(\frac{2}{\delta})}{2n}} + C\sqrt{\frac{\log{\frac{2}{\delta}}}{2n}}
	\end{align*}
\end{proposition}

\begin{proof}[Sketch of the proofs of Propositions \ref{prop:Lasso_Varsel}, \ref{prop:Lasso_Varsel_all}]
	Proof of both Propositions follows the same scheme as the Proposition  \ref{prop:Lasso_rad_gen}. Notice, that after selecting $s$ features the dimensionality of regression vector reduces to $s$, and thus the impact of the LASSO regularization in the complexity term will change from $\log{d}$ to $\log{s}$. Furthermore, for the Proposition  \ref{prop:Lasso_Varsel} we use bound from Lemma 1 combined with the bound on the log-cardinality from Lemma \ref{lem:card02} and general bound for a fixed class from Lemma  \ref{lem:rad_lasso_con}. In the same vein, to obtain the bounds from Proposition \ref{prop:Lasso_Varsel} we use the result of Lemma 1 combined with the bound on the log-cardinality from Lemma \ref{lem:card03} and general bound for a fixed class from Lemma  \ref{lem:rad_lasso_con}.   
\end{proof}

\subsection{Case of clipped loss function}
In this part we additionally investigate the influence of the scaling constants of the model class regularization constraints, under the assumption that the square loss function can be restricted to some bounded domain, when using some apriori knowledge independent of the constraints upper bounds on the target function. 
We give the definition (following \cite{steinwart2008support}, chapter 2) of the clipped loss function below.
\begin{definition}
	A loss function $\ell\paren{y,t}: (y,t) \mapsto \mathbb{R}$ can be clipped at the value $H>0$ if for all $\paren{y,t}$ we have: 
	\begin{align*}
	\ell\paren{y,\widetilde{t}} \leq \ell\paren{y,t},
	\end{align*}
	where $\widetilde{t}$ denotes the clipped value of $t$ at $H$, i.e. : 
	
	\begin{align*}
	\widetilde{t} = \begin{cases}
	-H & t < -H \\
	t &  t \in [-H,H] \\
	H & t > H,
	\end{cases}
	\end{align*}
	or equivalently $\widetilde{t} = \max\{-H,\min\{t,H\}\}$.
\end{definition}

Now let us consider the squared loss, i.e. $\ell(y,t) = \paren{y-t}^{2}$ and for the apriori given $M>0$ define the loss function $\widetilde{\ell}(y,t) = \min\paren{M,\paren{y-t}^{2}}$. Assuming, that the output variable $y$ has bounded support on $[-R,R]$ one can readily check, that the loss function $\ell(y,t)$ can be clipped at value $R+\sqrt{M}$ and that the loss function $\widetilde{\ell}(y,t)$ is its clipped version. Moreover, doing straightforward calculations, one can obtain, that the Lipschitz constant of $\tilde{\ell}\paren{y,\cdot}$ is bounded by $4R+2\sqrt{M}$.

Let firstly $\mathcal{F}$ be some model class, such that $\mathcal{F} = \cup_{i=1}^{N} \mathcal{F}_{i}$ and $\mathcal{F}_{i}$ be the arbitrary functional classes of real-valued functions with domain in $\mathcal{X}$. Applying Theorem \ref{thm:sup_bound} with the sets $A_{j}= \{\paren{\widetilde{\ell} \circ f(\bx_{1}),\ldots,\widetilde{\ell} \circ f(\bx_{n})}: f \in \mathcal{F}_{j} \}$, $j \in \{1,\ldots, N\}$ we obtain the following result for the Rademacher complexity $\hat{\mathfrak{R}}_{\mathbb{S}}(\widetilde{\ell} \circ \mathcal{F})$ of the image of the class $\mathcal{F}$ under the $\ell(\cdot,\cdot)$ map:
\begin{lemma}
	\label{lem:img_comp}
	\begin{align*}
	\hat{\mathfrak{R}}_{\mathbb{S}}(\widetilde{\ell} \circ \mathcal{F}) \leq \max_{m=1}^{N} \hat{\mathfrak{R}}_{\mathbb{S}}\paren{\widetilde{\ell} \circ \mathcal{F}_{m}} + 4 \mathcal{M} \sqrt{\frac{\log N}{n}},
	\end{align*}
	where we have $\mathcal{M} := \sqrt{\sup_{f \in \cup_{m=1}\mathcal{F}_{m}}\frac{1}{n}\sum_{i=1}^{n} \paren[2]{\widetilde{\ell} \circ f(\bx_{i})}^{2}}$
\end{lemma}

\textbf{Remark} Notice, that from the definition of the function $\widetilde{\ell}\paren{y,t}$ it follows that $\mathcal{M} \leq M$. Notice that in Lemma \ref{lem:img_comp} a constant $\mathcal{M}$ depends only on the apriori bound $M$ and on the contrary to that from the Propositions \ref{eq:rad_comp_ell_2} or \ref{prop:Lasso_rad_gen} where it scales linearly with the norm constraint of the prediction function $f$.

We demonstrate this more precisely  on the example of $\ell_{1}$-type regularization constraints below. Firstly, through contraction principle with Lipschitz constant $L:= 4R+2\sqrt{M}$ of function $\widetilde{\ell}(y,t)$ we deduce from Lemma \ref{lem:img_comp} that it holds: 
\begin{align*}
\hat{\mathfrak{R}}_{\mathbb{S}}(\widetilde{\ell} \circ \mathcal{F}) \leq L\max_{m=1}^{N} \hat{\mathfrak{R}}_{\mathbb{S}}\paren{\mathcal{F}_{m}} + 4 M \sqrt{\frac{\log N}{n}}.
\end{align*}
Thus, considering class $\mathcal{{F}}$ from Equation~\eqref{eq:regr_trees} with the norm constraints $B =\{f: \norm{f}_{1} \leq W \}$ and using the upper bound of the Lemma \ref{lem: help_lem01} and upper bound on the empirical Rademacher complexity of one tree with fixed structure and partition (Lemma \ref{lem:rad_lasso_con} ) we obtain.

\begin{proposition}[Rademacher and Generalization error bound for $\ell_{1}-$norm constraint with clipped squared loss]
	\label{prop:Lasso_rad__clipped}
	Let the function class $\mathcal{F}$ be as given in the equation \eqref{func_class001} with $B=\{\norm{f}_{1} \leq W\}$. Let also the underlying loss function be the clipped loss $\widetilde{\ell}$ of the squared loss, clipped at the point $R+\sqrt{M}$. Then the following (data-dependent) upper bounds for both the empirical and true Rademacher complexities of the class $\mathcal{F}$ is true: 
	\begin{align*}
	\hat{\mathfrak{R}}_{\mathbb{S}}\paren{\widetilde{\ell} \circ \mathcal{F}} &\leq \frac{\sqrt{\ell}}{\sqrt{n}} \paren{2^{\frac{1}{2}}WL\sqrt{\frac{1}{n}\sum_{j}\norm{\bx_{j}}_{\infty}^{2}}\paren{1 + 4 \sqrt{\log{d}}} +4M \sqrt{\log\paren{enD}}} \\ 
	{\mathfrak{R}}_{n}\paren{\widetilde{\ell} \circ \mathcal{F}} &\leq \frac{\sqrt{\ell}}{\sqrt{n}} \paren{ 2^{\frac{1}{2}}WL\sqrt{\ee{}{\norm{\bx}_{\infty}^{2}}}\paren{1 + 4 \sqrt{\log{d}}} + 4M \sqrt{\log\paren{enD}}}
	\end{align*}
	Furthermore, with probability at least $1-\delta/2$ w.r.t. sample $\mathbb{S}$ we have for all $f \in \mathcal{{F}}$: 
	\begin{align*}
	L(f) - L_{n}(f)\leq  & + \frac{2\sqrt{\ell}}{\sqrt{n}} \paren{ 2^{\frac{1}{2}}WL\sqrt{\ee{}{\norm{\bx}_{\infty}^{2}}}\paren{1 + 4 \sqrt{\log{d}}} + 4M \sqrt{\log\paren{enD}}} + C\sqrt{\frac{\log(\frac{2}{\delta})}{2n}}.
	\end{align*}
	where $C = (R+WK)^{2}$.
\end{proposition}

\textbf{Remark. } We observe that in the previous bound for Rademacher complexity only the first term scales linearly with the norm constraint, but the second term scales only with the constant $M$ (which apriori does not depend on the norm constraint $W$).

The effect of the clippable loss can be extended to the setting of feature selection procedure, when it is performed at each internal node and in the leaves. Recall that in this case the underlying model class can be written as: 
\begin{align}
\label{eq:class_str_all_sup}
\mathcal{F}_{\ell,sel} = \bigcup_{\substack{T \in \cT_l\\(i_{k},t_{k})_{k} \in ([D] \times \mbr)^{\cT^\circ} \\ \paren{\overline{s}}_{p} \in \overline{S}^{\otimes 2\ell-1}}}\mathcal{F}_{T,\paren{(i_{k},t_{k})_{k\in T^{\circ}}}, \paren{\overline{s}}_{p}}.
\end{align}
The following result for the clipped square-loss function holds true. 

\begin{proposition}
	\label{prop:Lasso_Varsel_all_clipped}
	Let function class $\mathcal{F}_{\ell,sel}$ be the class of decision trees with regularization constraints $B=\{\norm{f}_{1} \leq W\}$ described by Equation \ref{eq:class_str_all_sup}. Then the following upper bound for the Rademacher complexity of the class $\mathcal{F}_{\ell,sel}$ is true: 
	\begin{align*}
	\hat{\mathfrak{R}}_{\mathbb{S}}(\mathcal{F}_{s}) \leq 
	\frac{\sqrt{\ell}WL}{\sqrt{n}} \paren[4]{2^{\frac{1}{2}} \sqrt{\frac{1}{n} \sum_{i=1}^{n}\norm{\bx_{j}}^{2}_{\infty}}\paren{1 + 4\sqrt{\log(s)}}} + \frac{4M\sqrt{\ell}}{\sqrt{n}} \sqrt{\log\paren{neD}} + \frac{8\sqrt{s}}{\sqrt{n}}M\sqrt{\log\paren{\frac{de}{s}}} .
	\end{align*}
	Also, with probability at least $1-\delta$ we obtain that for all $f \in \mathcal{{F}}_{\ell,sel}$ it holds:
	\begin{align*}
	L(f) - L_{n}(f)&\leq   + \frac{2\sqrt{\ell}WL}{\sqrt{n}} \paren[4]{2^{\frac{1}{2}} \sqrt{\ee{}{\norm{\bx}^{2}_{\infty}}}\paren{1 + 4\sqrt{\log(s)}}} + \frac{8M\sqrt{\ell}}{\sqrt{n}} \sqrt{\log\paren{neD}} + \frac{16\sqrt{s\ell}}{\sqrt{n}}M\sqrt{\log\paren{\frac{de}{s}}} \\ &+ M\sqrt{\frac{\log{\frac{2}{\delta}}}{2n}},
	\end{align*}
	where $L= 4R + 2\sqrt{M}$ is the Lipschitz constant as before. 
\end{proposition}

\textbf{Remark.} Notice that the three terms in the last bound on the generalization error in the right hand side have the same scaling in terms of $\sqrt{\ell}/\sqrt{n}$, but only the first term scales linearly with respect to norm constraints $W$. Also, for $n$ large and $d<<D$, second term dominates the bound, but we can balance (and thus make the bound sharper) between first and third terms by choosing the number of selected variables depending on the norm constraints. More precisely,  let $A:= \frac{2\sqrt{\ell}WL}{\sqrt{n}} \paren[4]{2^{\frac{1}{2}} \sqrt{\ee{}{\norm{\bx}^{2}_{\infty}}}\paren{1 + 4\sqrt{\log(s)}}}$ and $B := \frac{16\sqrt{s\ell}}{\sqrt{n}}M\sqrt{\log\paren{\frac{de}{s}}}$. Solving equation $A=B$ in $s$ for fixed norm constraint $W$ we have
\begin{align}
\label{eq:s_optimal_scale}
s^{\star} = \Psi^{-1}\paren{\frac{WL\sqrt{\ee{}{\norm{\bx}_{\infty}^{2}}}}{2^{\frac{5}{2}}}},
\end{align}
where $\Psi\paren{s} = \frac{\sqrt{s\log{\frac{de}{s}}}}{1+4\sqrt{\log{s}}}$ and $\Psi^{-1}$ is its inverse. This provides us the choice of number $s$ of variables to select, depending on the input dimension $d$, constant $L$ and norm constraints $W$ that balances the statistical bound for generalization error from the Proposition \ref{prop:Lasso_Varsel_all_clipped}. Notice that we can also obtain purely data-dependent selection rule, when substituting $\ee{}{\norm{\bx}}$ with its empirical counterpart $\frac{1}{n}\sum_{i=1}^{n}\norm{\bx_{i}}^{2}_{\infty}$.

\section{Empirical Evaluations}

\subsection{Datasets}
\begin{itemize}
	\item {\bf KDD-cup 2004} this data set which was part of  the KDD-cup   2004   competition   data. In particular
	we use the data from the physics competition which comprises data on 78 
	properties of 150.000 particles. As in \cite{vogel07scalable}, we use the value
	of the 24th column as our target value. 
	\item{\bf Forest}, here the task is to predict the forest cover type from $54$ cartographic variables. As in \cite{Ronan}, we create a regression task
	by assigning a label of  $+1$ to samples corresponding to cover type $2$ and $-1$ to the rest. As in \cite{Ronan}, we use $25.000$ samples for training and another $10.000$ for testing.
	\item{\bf CT-Slice}  comprises $53500$ CT images, from $74$ different patients, represented by two histograms, one for bone structure and one for air inclusions.
	The dimensionality of the data is $383$. The target value of the task is the relative location on of the image on the axial axis. We assign images from two thirds of the patients to the training datasets, while the remaining third
	we keep as a test set.
	\item{\bf MNIST}  consists of $60.000$ training samples and $10.000$ test samples of dimensionality $784$, representing handwritten digits. In order to create a 
	regression task, we assign a label of $+1$ to all odd digits and $-1$ to the even digits.
	Following \cite{Ronan} we create a second regression task on the MNIST dataset by assigning a label of $+1$ to the digits $0$ to $4$ and $-1$ to digits $5$ through $9$. 
	\item{\bf Kinematic} which comes from a realistic simulation of the dynamics of an 8 link robot arm. 
	The task is to predict the distance of the arm's end-effector from some target.  We use
	two versions of this dataset, one with medium noise and another with high noise. In both cases the dimensionality of the data is 32. 
	\item{\bf Energy}  The regression task here is to predict the appliances energy consumption (in Wh) of a household given the 28 measurements of a wireless sensor network of the temperature and
	humidity of various rooms in the household. The dataset's github repository provides a train/test split comprising 14,803 training samples and 4932 testing samples.
	\item{\bf Air} The dataset contains 9358 instances of hourly averaged responses from an array of 5 metal oxide chemical sensors embedded in an Air Quality Chemical Multisensor Device. We
	randomly subsampled 6000 samples for training and retained the rest for testing. 
	
\end{itemize}
\subsection{Effects of Speedups}

We present here an empirical analysis of the effects on the prediction error of the two approximation speedups proposed in the paper. Recall that we considered two settings for the approximation. 

\begin{enumerate}
	\item  $\forall m, k \leq m \leq N-k$ we approximate 
	\begin{equation*}
	L^\lambda_{i,m} \approx l^\lambda_{j,k} + r^\lambda_{j,N-k} + \left(N-2*k\right) \left(\min(l^\lambda_{j,k}-\lambda \| {w}_{j,k} - {w}_0\|^2_Q,r^\lambda_{j,N-k}-\lambda \| {w^c}_{j,N-k} - {w^c}_0\|^2_Q)\right).
	\end{equation*}
	If for a given $k$ , $L^\lambda_{i,m} \geq L^\lambda_{i_T,k_T}$, we forgo calculating $l^\lambda_{j,m},r^\lambda_{j,N-m}, \forall m, k \leq m \leq N-k$.
	
	\item   $\forall m, k \leq m \leq N-k$ we approximate 
	\begin{equation*}
	L^\lambda_{i,m} \approx l^\lambda_{j,k} + r^\lambda_{j,N-k} + \left(N-2*k\right) \left(\max(l^\lambda_{j,k}-\lambda \| {w}_{j,k} - {w}_0\|^2_Q,r^\lambda_{j,N-k}-\lambda \| {w^c}_{j,N-k} - {w^c}_0\|^2_Q)\right).
	\end{equation*}
	If for a given $i$ , $L^\lambda_{i,m} \geq L^\lambda_{i_T,k_T}$, we forego calculating $l^\lambda_{j,m},r^\lambda_{j,N-m}, \forall m, k \leq m \leq N-k$.
\end{enumerate}

We set regularization parameter $\gamma=1$  throughout these experiments and plot the loss for various tree depths on the various datasets for the 3 type of algorithmic procedures (exact, approximate (1.), approximate (2.)). 

\begin{figure}[H]
	\begin{minipage}[h]{0.32\linewidth}
		\center{\includegraphics[width=1\linewidth,scale=0.5]{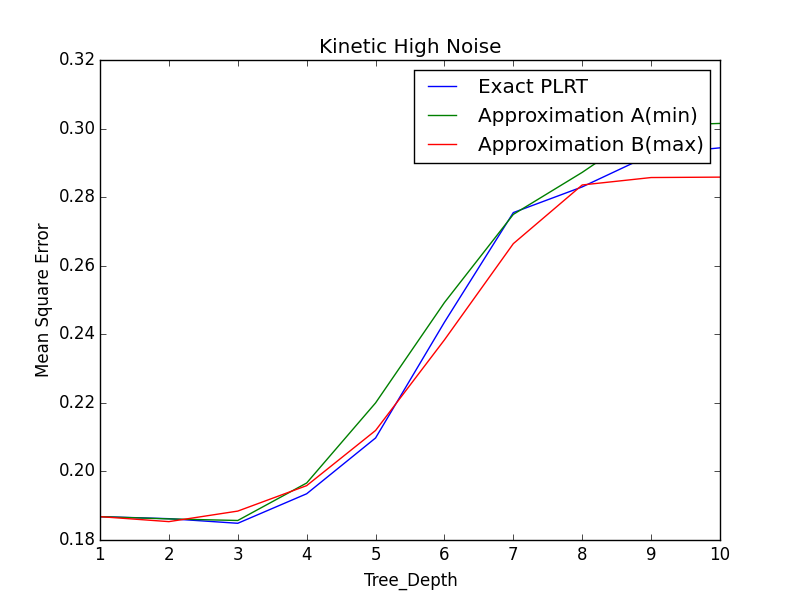}} \\
	\end{minipage}
	\hfill
	\begin{minipage}[h]{0.32\linewidth}
		\center{\includegraphics[width=1\linewidth,scale=0.5]{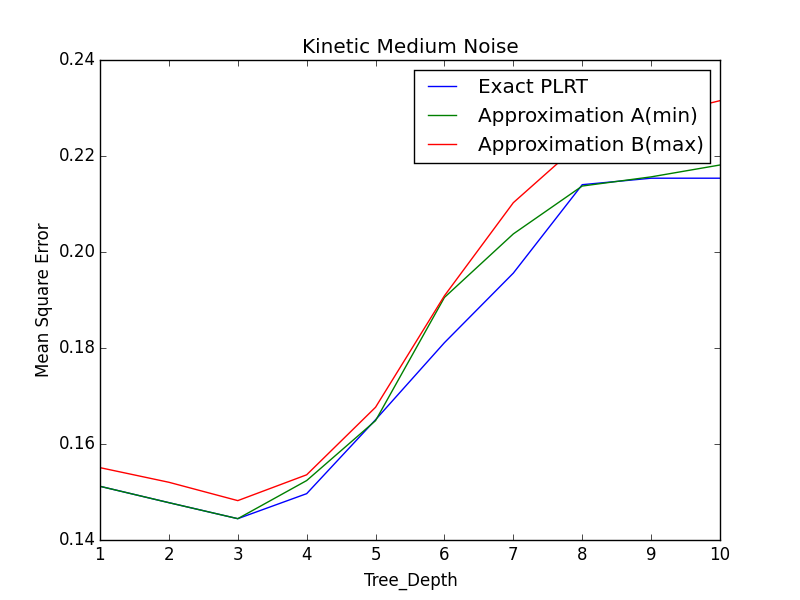}} \\
	\end{minipage}
	\hfill
	\begin{minipage}[h]{0.32\linewidth}
		\center{\includegraphics[width=1\linewidth,scale=0.5]{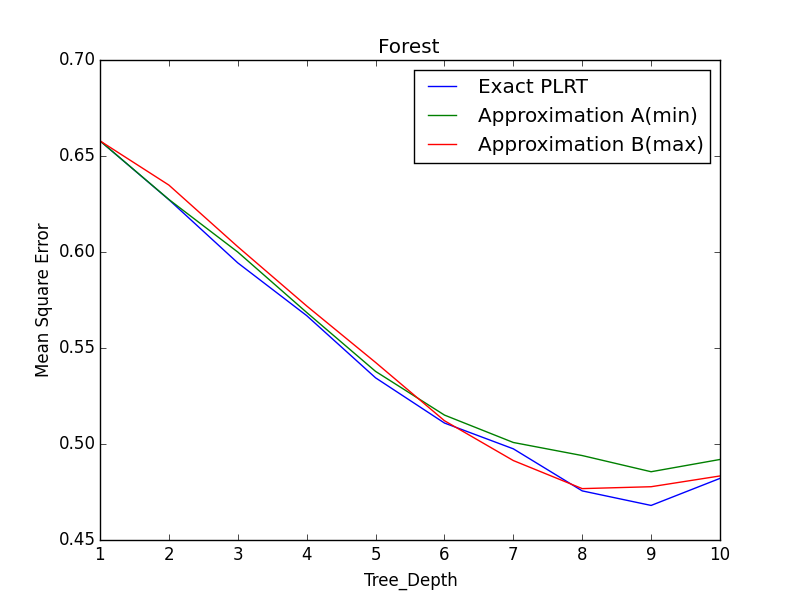}} \\
	\end{minipage}
	\vfill
	\begin{minipage}[h]{0.32\linewidth}
		\center{\includegraphics[width=1\linewidth,scale=0.5]{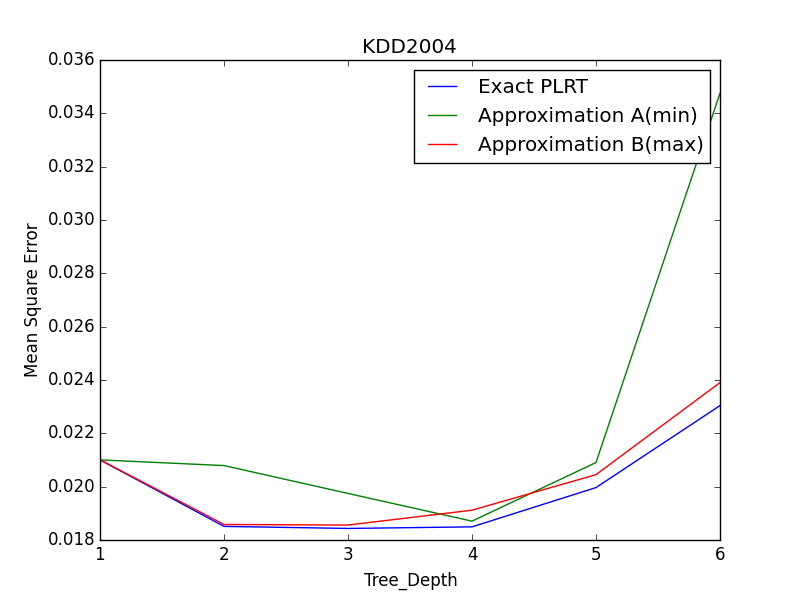}} \\
	\end{minipage}
	\hfill
	\begin{minipage}[h]{0.32\linewidth}
		\center{\includegraphics[width=1\linewidth,scale=0.5]{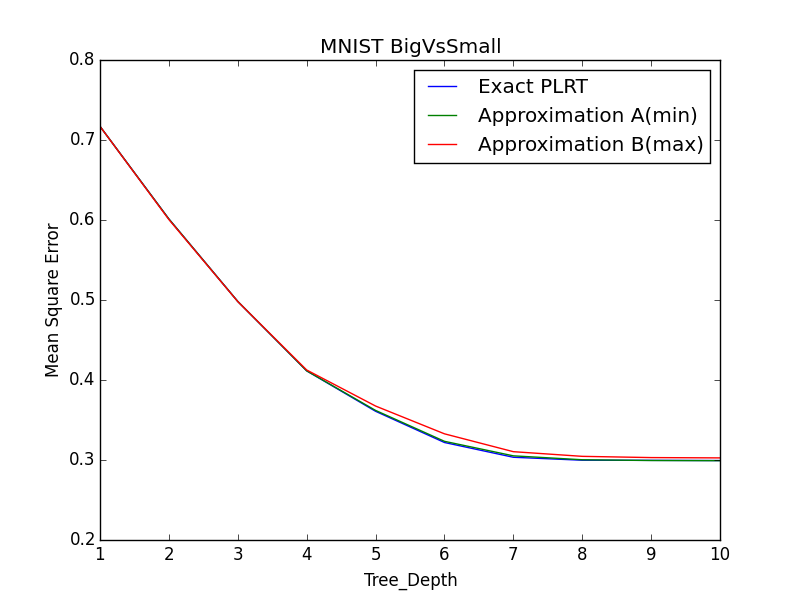}} \\
	\end{minipage}
	\hfill
	\begin{minipage}[h]{0.32\linewidth}
		\center{\includegraphics[width=1\linewidth,scale=0.5]{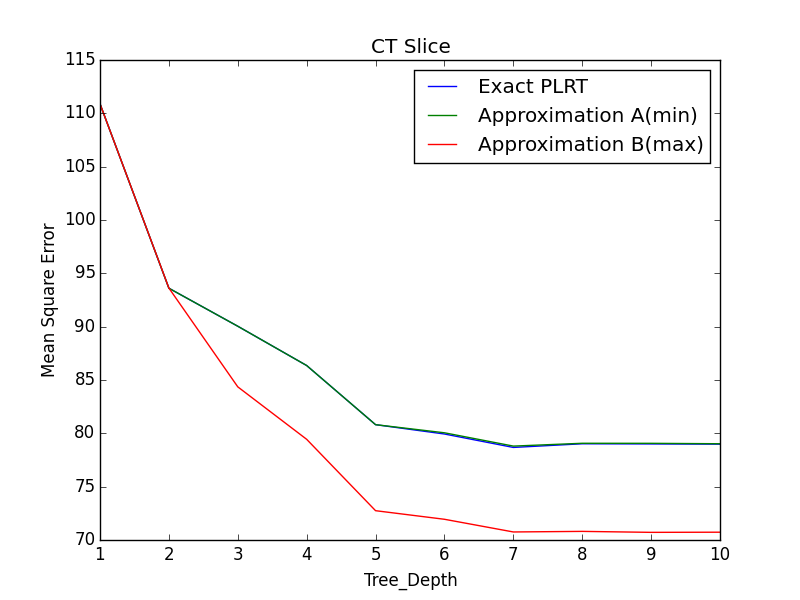}} \\
	\end{minipage}
	\vfill 
	\begin{minipage}[h]{0.32\linewidth}
		\center{\includegraphics[width=1\linewidth,scale=0.5]{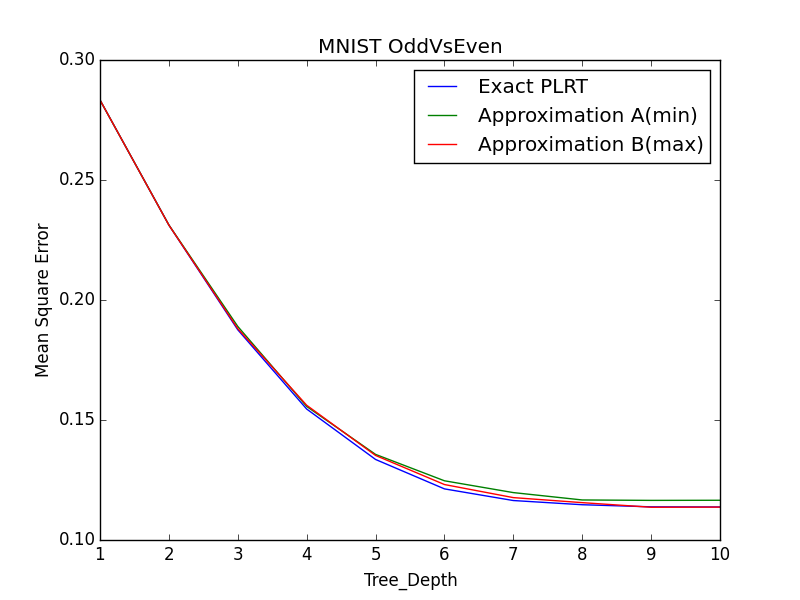}} \\
	\end{minipage}
	\begin{minipage}[h]{0.32\linewidth}
		\center{\includegraphics[width=1\linewidth,scale=0.5]{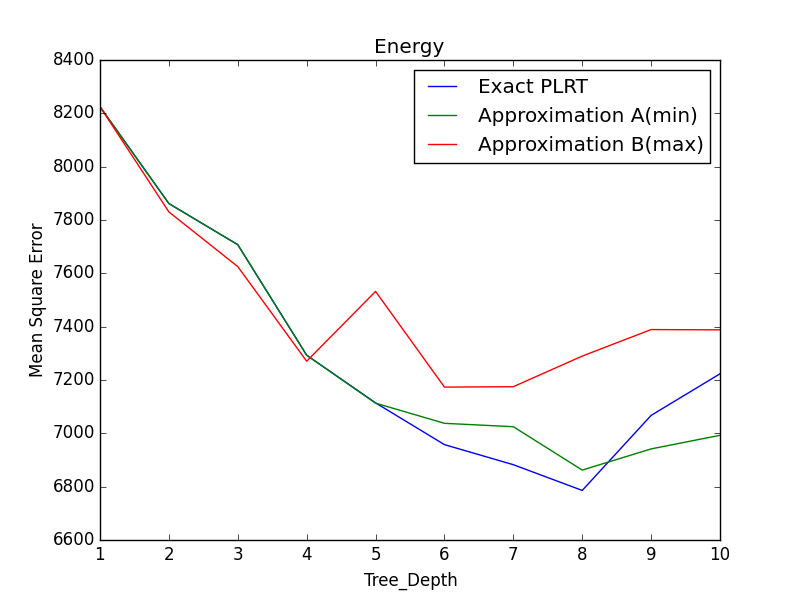}} \\
	\end{minipage}
	\begin{minipage}[h]{0.32\linewidth}
		\center{\includegraphics[width=1\linewidth,scale=0.5]{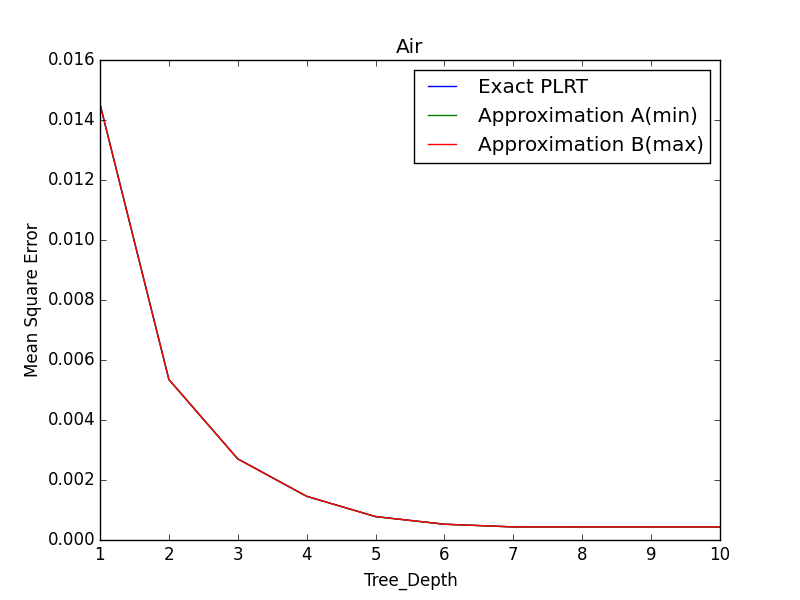}} \\
	\end{minipage}
	\caption{\label{fig:RES1} The mean squared error of the three speedup algorithms. }
\end{figure}

As can be seen in most cases the difference in accuracy is negligible. Only in two of the Datasets (Energy, KDD2004) the approximation algorithms lead to poorer performance. Given
the speedups they offer in computation time, we suggest that there is considerable merit in employing these algorithms despite the fact that they are not exact.

\subsection{Empirical effects of the choice of parameter $\lambda$ ($\ell_{1}-$penalization)}

We present here an empirical analysis of the effect on generalization of using LASSO in the leaves of a PLRT tree built with $l_2$-regularization in the nodes. In particular we set $\gamma=1.0$
constantly throughout all these experiments meaning that all the trees were built by optimizing $\left\| \bX_A w - Y_A \right\|^2 +  \left\| w-w_0\right\|^2$, and evaluate 
the performance of the trees for various values of the LASSO regularization parameter $\lambda$.

\begin{figure}[H]
	\begin{minipage}[h]{0.32\linewidth}
		\center{\includegraphics[width=1\linewidth,scale=0.5]{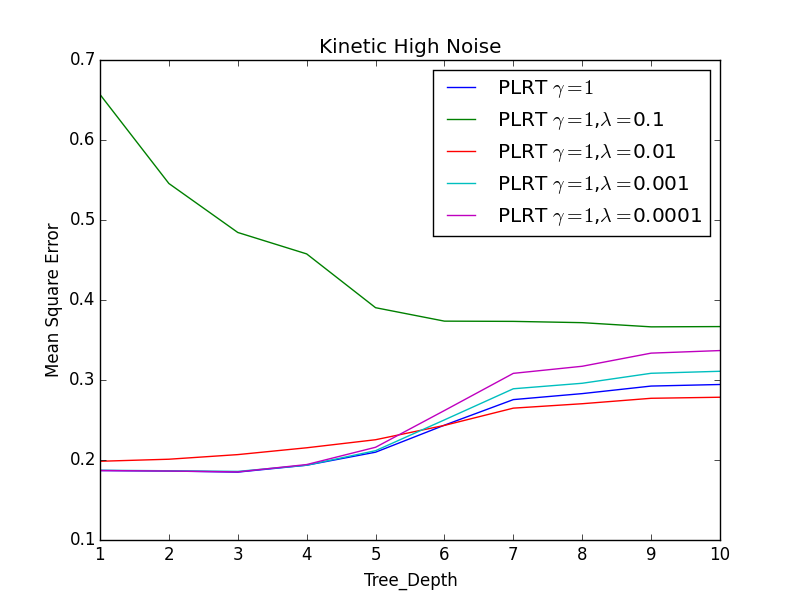}} \\
	\end{minipage}
	\hfill
	\begin{minipage}[h]{0.32\linewidth}
		\center{\includegraphics[width=1\linewidth,scale=0.5]{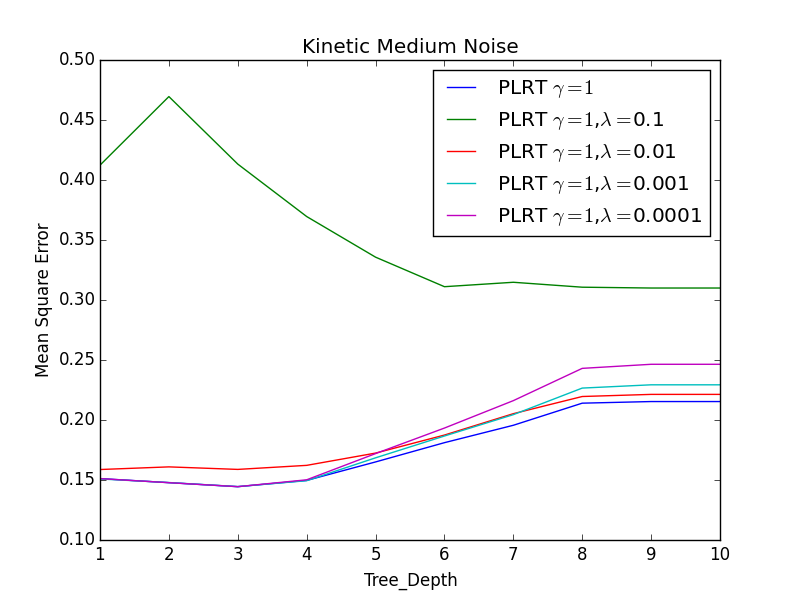}} \\
	\end{minipage}
	\hfill
	\begin{minipage}[h]{0.32\linewidth}
		\center{\includegraphics[width=1\linewidth,scale=0.5]{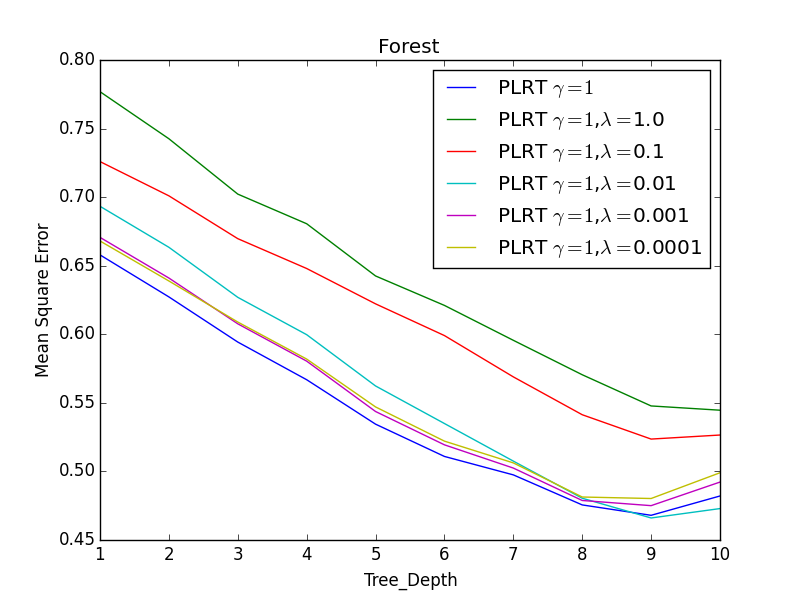}} \\
	\end{minipage}
	\vfill
	\begin{minipage}[h]{0.32\linewidth}
		\center{\includegraphics[width=1\linewidth,scale=0.5]{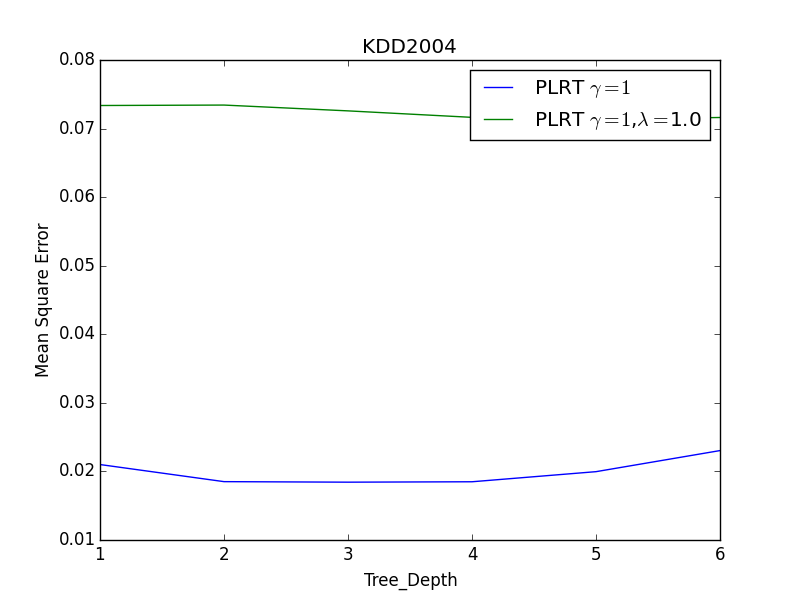}} \\
	\end{minipage}
	\hfill
	\begin{minipage}[h]{0.32\linewidth}
		\center{\includegraphics[width=1\linewidth,scale=0.5]{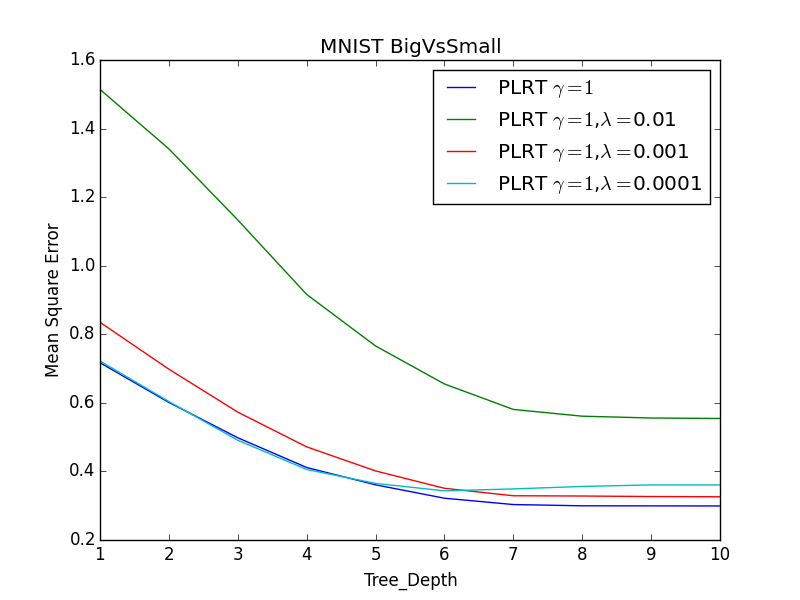}} \\
	\end{minipage}
	\hfill
	\begin{minipage}[h]{0.32\linewidth}
		\center{\includegraphics[width=1\linewidth,scale=0.5]{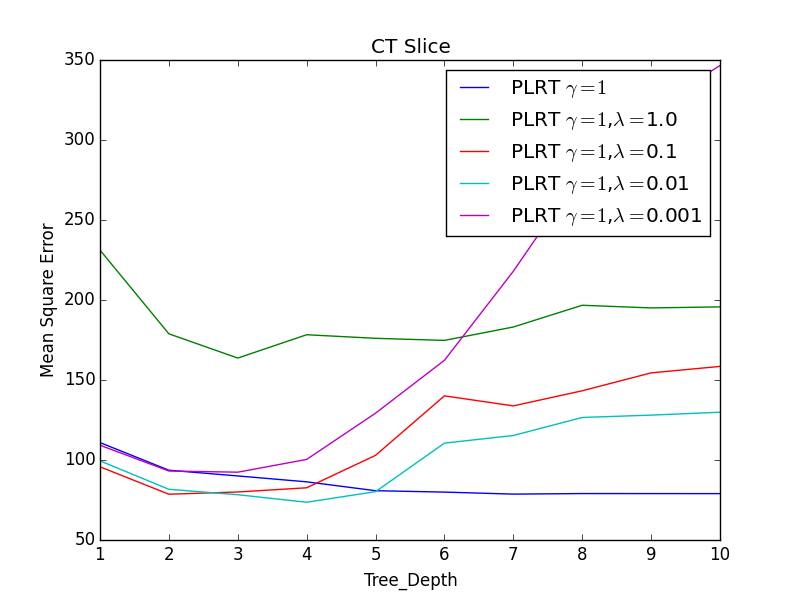}} \\
	\end{minipage}
	\vfill 
	\begin{minipage}[h]{0.32\linewidth}
		\center{\includegraphics[width=1\linewidth,scale=0.5]{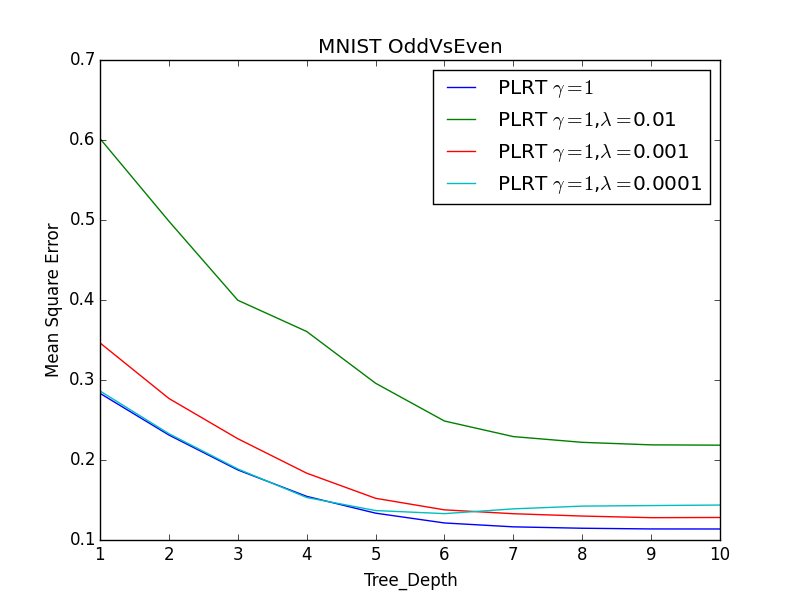}} \\
	\end{minipage}
	\begin{minipage}[h]{0.32\linewidth}
		\center{\includegraphics[width=1\linewidth,scale=0.5]{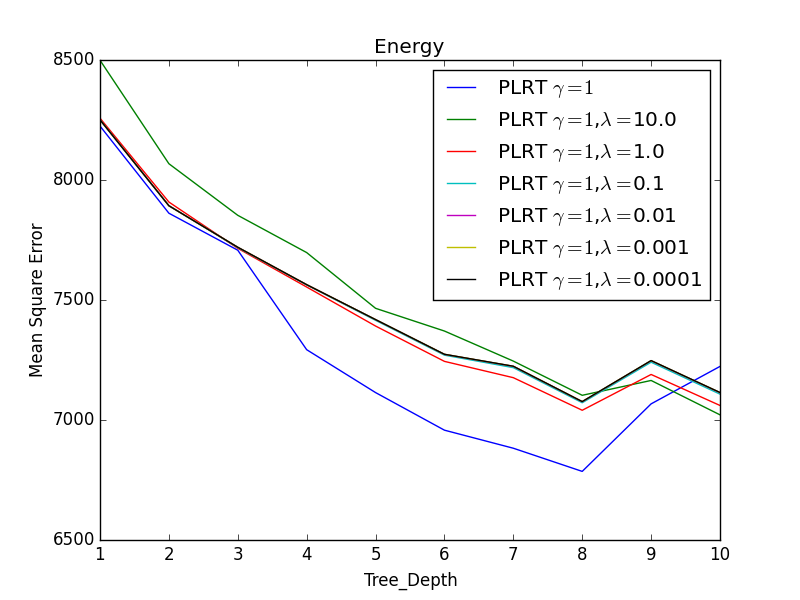}} \\
	\end{minipage}
	\begin{minipage}[h]{0.32\linewidth}
		\center{\includegraphics[width=1\linewidth,scale=0.5]{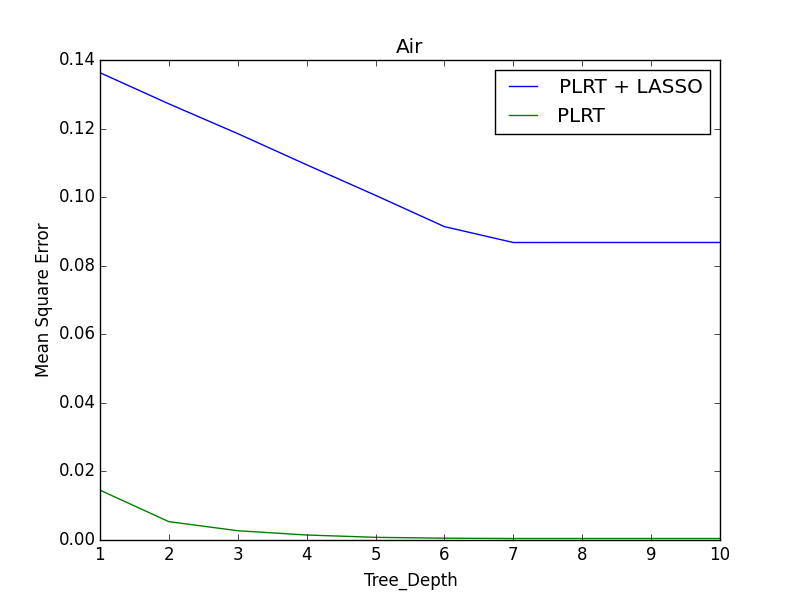}} \\
	\end{minipage}
	\caption{\label{fig:RES2} The mean square error for various values of the $l_1$-regularization parameter $\lambda$, for $\gamma=1.0$.}
\end{figure}

As can be seen, LASSO does not lead to improved performance on the datasets presented here. This may be related to the difference, when using LASSO, in 
regularization criteria in the nodes and the leaves. The tree itself is constructed by using an $l_2$-regularization forcing the optimization algorithms in the prospective
leaves to find weight vector solution close (in an $l_2$ sense) to the optimal weight vector already calculated in the node. This form of regularization played a crucial role
in the development of the presented PLRT algorithm. As can be seen in the next section, strong regularization, i.e. a large $\gamma$ value, is needed for the stability of the algorithm.

\subsection{Effects of parameter $\gamma$ ($l_2$ regularization)}
We present here an empirical evaluation of the effect of the $\gamma$ on generalization, for the various datasets. As can be seen the proposed algorithm is prone to overfitting for small 
values of $\gamma$. In fact for almost all datasets the algorithm overfits even for moderately deep trees. We surmise that strong regularization $\gamma\geq 1.0$ is crucial to the
performance of the algorithm. By propagating the weight vectors of the higher nodes, through the $l_2$ regularization, the space of viable weight vector solutions is constrained and
overfitting is avoided. As noted this was a key insight in developing the proposed solution. 
\begin{figure}[H]
	\begin{minipage}[h]{0.32\linewidth}
		\center{\includegraphics[width=1\linewidth,scale=0.5]{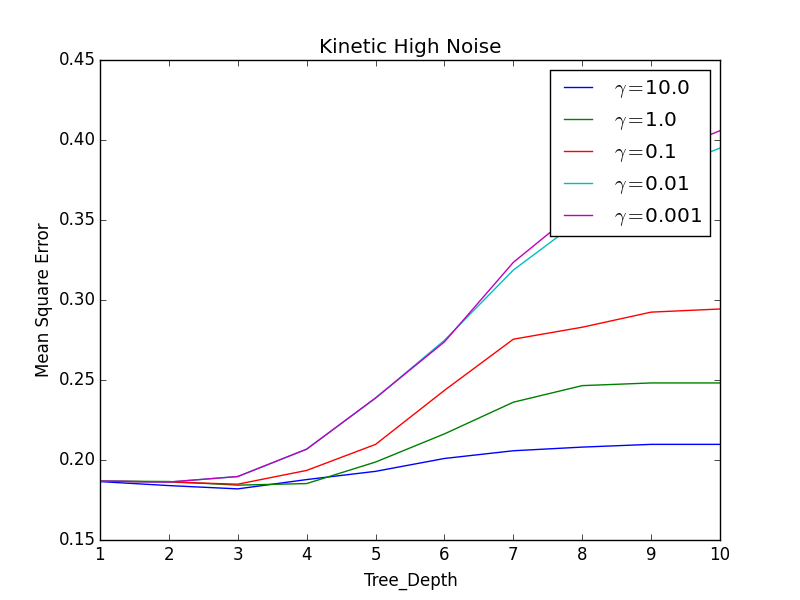}} \\
	\end{minipage}
	\hfill
	\begin{minipage}[h]{0.32\linewidth}
		\center{\includegraphics[width=1\linewidth,scale=0.5]{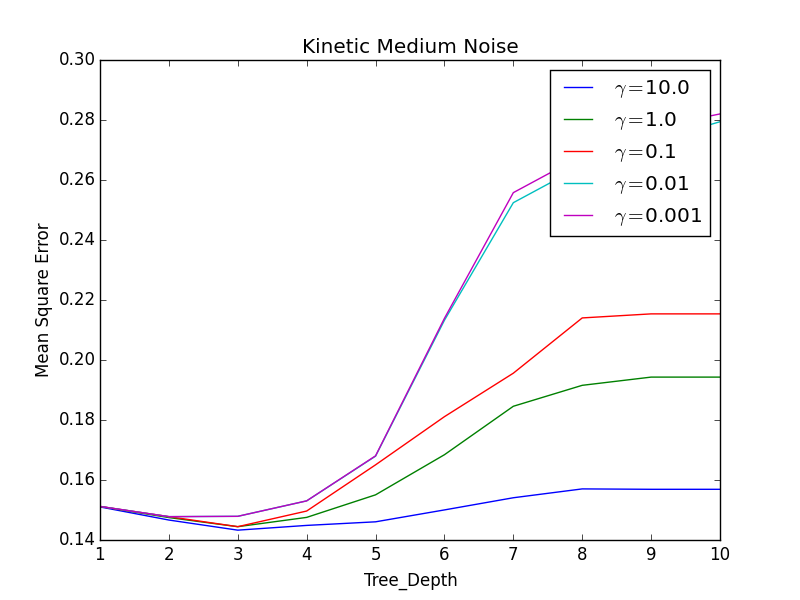}} \\
	\end{minipage}
	\hfill
	\begin{minipage}[h]{0.32\linewidth}
		\center{\includegraphics[width=1\linewidth,scale=0.5]{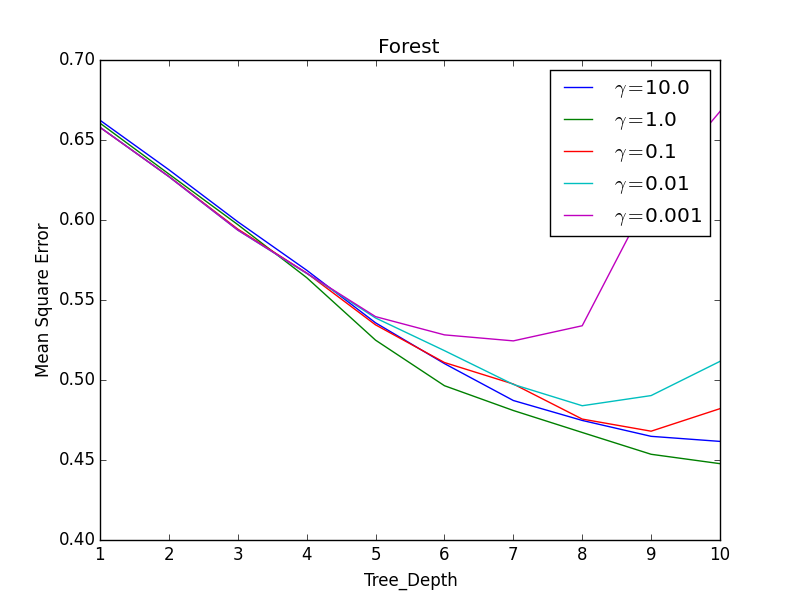}} \\
	\end{minipage}
	\vfill
	\begin{minipage}[h]{0.32\linewidth}
		\center{\includegraphics[width=1\linewidth,scale=0.5]{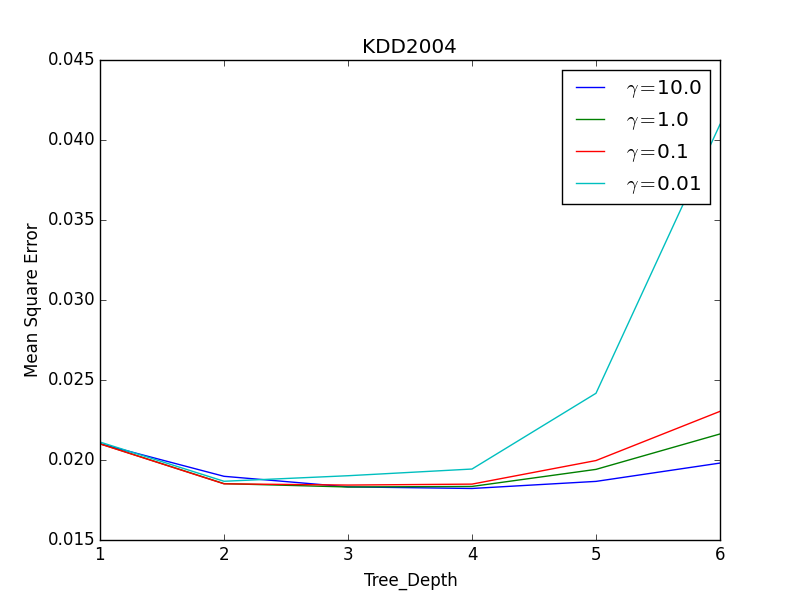}} \\
	\end{minipage}
	\hfill
	\begin{minipage}[h]{0.32\linewidth}
		\center{\includegraphics[width=1\linewidth,scale=0.5]{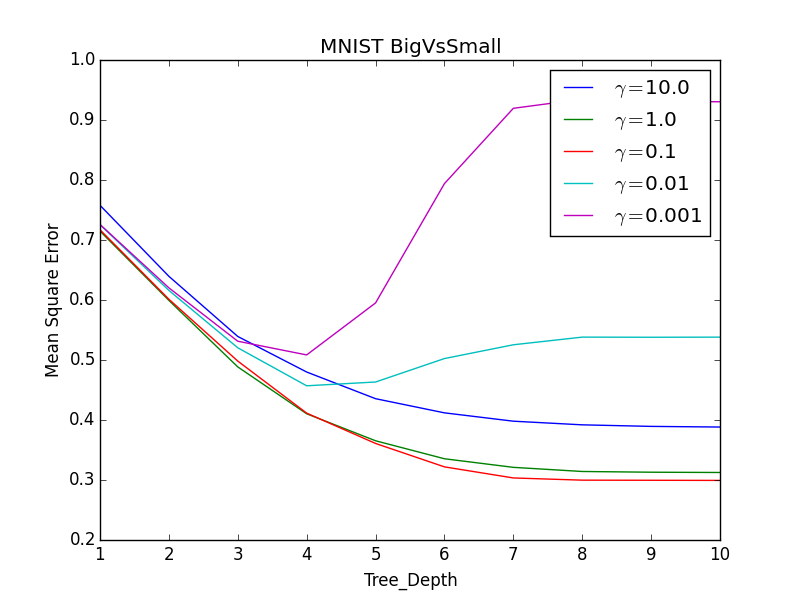}} \\
	\end{minipage}
	\hfill
	\begin{minipage}[h]{0.32\linewidth}
		\center{\includegraphics[width=1\linewidth,scale=0.5]{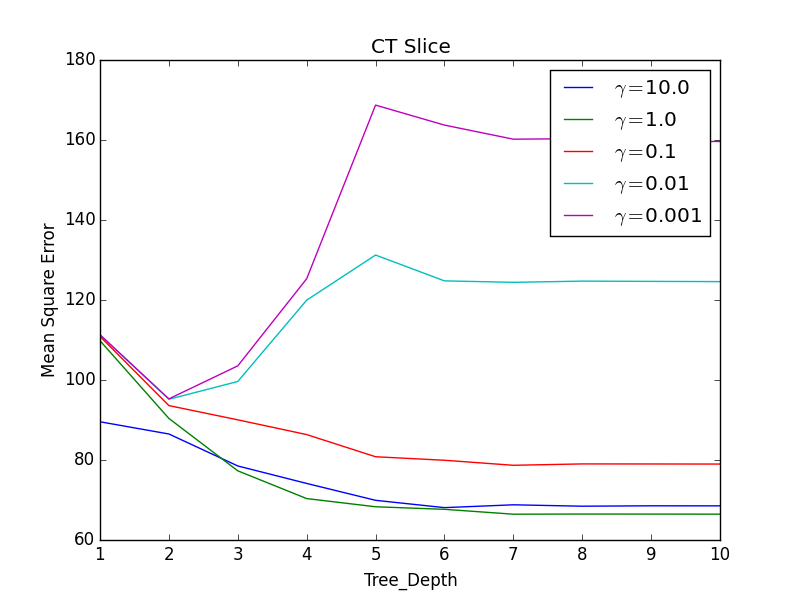}} \\
	\end{minipage}
	\vfill 
	\begin{minipage}[h]{0.32\linewidth}
		\center{\includegraphics[width=1\linewidth,scale=0.5]{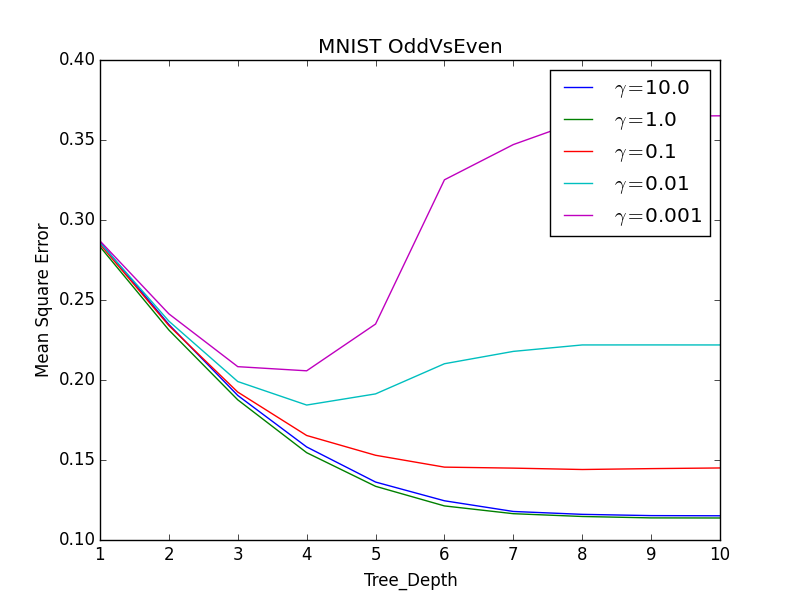}} \\
	\end{minipage}
	\begin{minipage}[h]{0.32\linewidth}
		\center{\includegraphics[width=1\linewidth,scale=0.5]{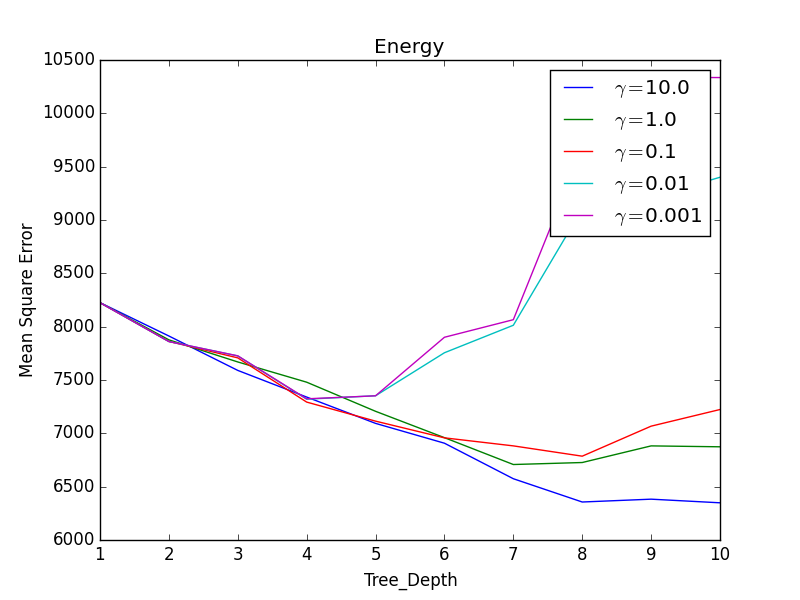}} \\
	\end{minipage}
	\begin{minipage}[h]{0.32\linewidth}
		\center{\includegraphics[width=1\linewidth,scale=0.5]{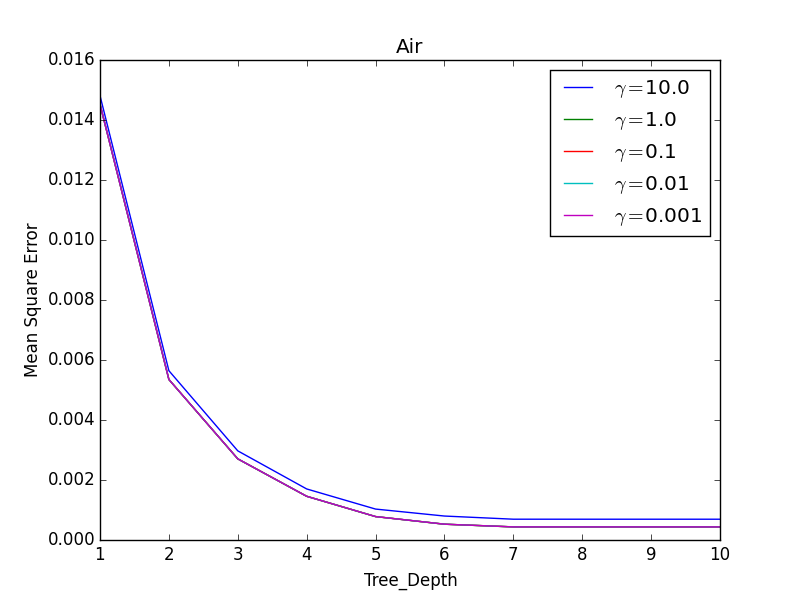}} \\
	\end{minipage}
	\caption{\label{fig:RES3}The mean squared error of the 3 predictors for various tree depths. }
\end{figure}

\section{Numerical Stability}

As noted, one concern that can be raised regarding the proposed method is its numerical stability. the proposed method involves calculating, for each threshold, the
inverse of a matrix of the form $\bX^T P \bX + \lambda Q$, where $\bX$ is a $N \times d$ matrix, meaning
$\bX^T P \bX + \lambda Q$ is the result of $N$ rank-one updates to $Q$. 

We consider in the following that $P,Q$ are identity matrices and investigate the
numerical stability of calculating $\Lambda = \left( \bX^T \bX + I \right)^{-1}$ as $N$ rank-updates using Sherman-Morrison $(\Lambda_{R1})$ and calculating the inverse from scratch 
each time via the Cholesky decomposition $(\Lambda_{CH})$. In particular we consider the relative error of the Frobenius norm
$\frac{\|\Lambda_{CH} - \Lambda_{R1}\|_F}{\|\Lambda_{CH}\|_F}$, where $\| \Lambda \|_F = \sum\limits_{i,j} \Lambda^2_{ij}$. For randomly generated data $\bX$, we plot
in Figure~\ref{fig:stab}(A) the error for various values of $N$ and $d=2048$. As can be seen the relative error in the Frobenius norm remains very small (of the order of magnitude $10^{-4}$) even for large
values of $N$. 

We furthermore investigate the numerical stability of the calculated vector $w = \left( \bX^T \bX + I \right)^{-1}\left( \bX^T Y\right)$, as this is ultimately what is
used to calculate the optimal split. In Figure~\ref{fig:stab}(B) we plot the angle, in degrees, between the vectors $w_{CH}$ and $w_{R1}$ calculated using Cholesky decompositions and $N$ rank-one updates
respectively, as before we plot the angle for various values of $N$ and for $d=2048$. As can be seen there is some inaccuracy in a certain bandwidth of values of $N$ of the order $O(d)$ but 
even in this case the angle between the two vectors is very small $\left(0.03^{\circ}\right)$. Plotting, in Fig.~\ref{fig:stab}(C) the conditioning number of the matrix 
$\kappa\left( \Lambda_{CH}\right) = \| \Lambda_{CH}^{-1}  \| \| \Lambda_{CH} \|$ against $N$ reveals the source of this, relatively small, numerical inaccuracy. The experiments were run using the cuBLAS library on the same type of GPU as used in the experiments (Tesla K80) .

\begin{figure}
	\begin{center}
		\begin{tabular}{ccc}
			\includegraphics[width=0.33\textwidth]{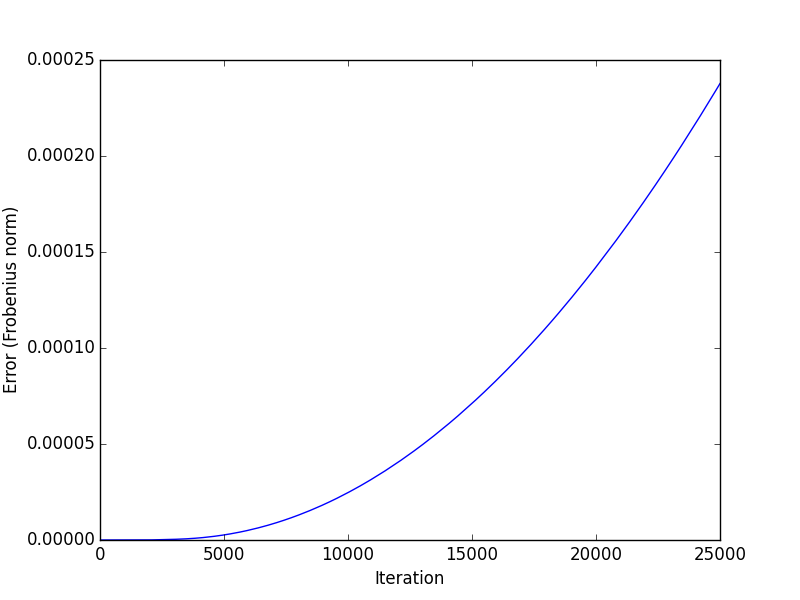} & \includegraphics[width=0.33\textwidth]{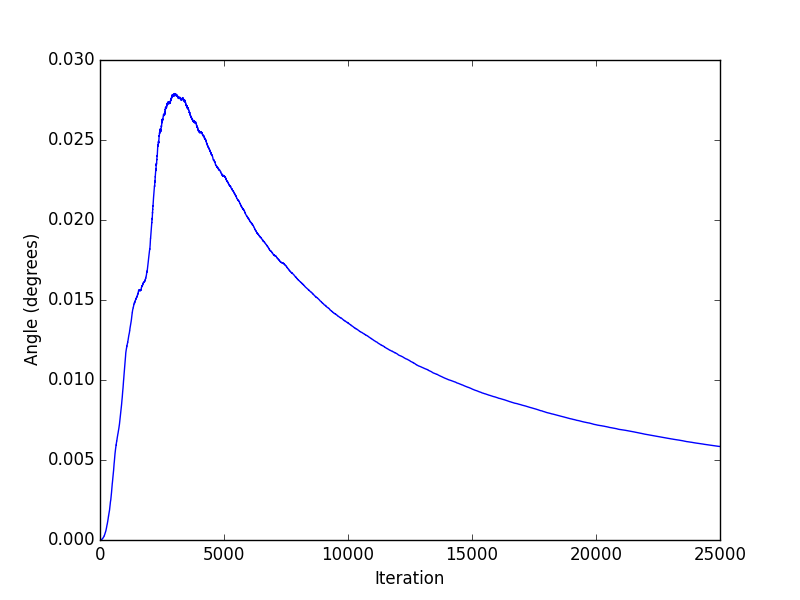} &  \includegraphics[width=0.33\textwidth]{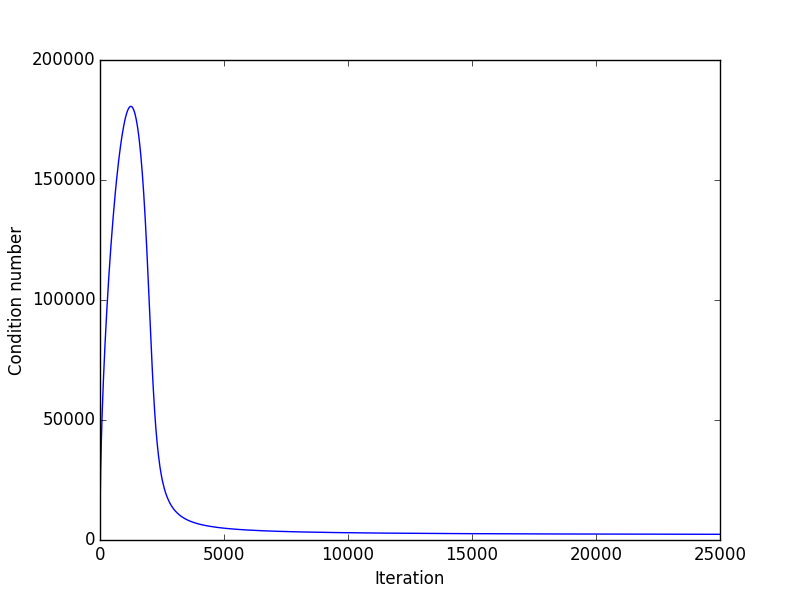}\\
			(A) & (B) & (C)
		\end{tabular}
	\end{center}
	\caption{{\label{fig:stab}} Numerical stability of calculating $\Lambda = \left( \bX^T \bX + I \right)^{-1}$ and $w = \left( \bX^T \bX + I \right)^{-1}\left( \bX^T Y\right)$
		via $N$ rank one updates to $I$  compared to calculating the same quantities via the Cholesky Decomposition: A) The relative error in the Frobenius norm, B) the angle, in degrees,
		between the vectors $w_{CH}$ and $w_{R1}$, and C) the conditioning number of $\kappa\left( \Lambda_{CH}\right) = \| \Lambda_{CH}^{-1}  \| \| \Lambda_{CH} \|$,
		for $d=2048$ and various values of~$N$}
\end{figure}
\end{document}